\documentclass[twoside]{article}

\usepackage{appendix}
\usepackage{bm}
\usepackage[accepted]{aistats2024}
\usepackage[round]{natbib}

\usepackage{xcolor}
\usepackage{amssymb}
\usepackage{amsthm}  
\usepackage{graphicx}
\usepackage{hyperref}

\usepackage{amsfonts}

\newtheorem{definition}[]{Definition}
\newtheorem{lemma}[]{Lemma}
\newtheorem{theorem}[]{Theorem}
\newtheorem{proposition}{Proposition}

\usepackage{fontawesome5}
\usepackage{tikz}
\newcommand{\cirbd}{\mathbin{\tikz{\draw[line width=0.25pt] (0,0) circle[radius=0.7ex];\draw[fill] (0,0) circle[radius=0.3ex];}}}

\newcommand{\calM}{\mathcal{M}}
\newcommand{\defeq}{:=}

\newcommand{\gauss}{\mathcal{N}}
\newcommand{\norm}[1]{\| #1 \|}
\newcommand{\prob}{\mathbb{P}}

\usepackage{amsmath}
\DeclareMathOperator*{\argmin}{arg\,min}

\begin{document}

%

%

\twocolumn[

\aistatstitle{Implicit Bias in Noisy-SGD: With Applications to Differentially Private Training}

\aistatsauthor{ Tom Sander \And Maxime Sylvestre \And Alain Durmus }

\aistatsaddress{Meta AI (FAIR) \& École polytechnique\textsuperscript{*}  \And  Université Paris Dauphine \And École polytechnique\textsuperscript{*}} ]

\begin{abstract}
\vspace{-0.3cm}
Training Deep Neural Networks (DNNs) with small batches using Stochastic Gradient Descent (SGD) yields superior test performance compared to larger batches. 
The specific noise structure inherent to SGD is known to be responsible for this implicit bias.
DP-SGD, used to ensure differential privacy (DP) in DNNs' training, adds Gaussian noise to the clipped gradients. 
Surprisingly, large-batch training still results in a significant decrease in performance, which poses an important challenge because strong DP guarantees necessitate the use of massive batches.
We first show that the phenomenon extends to Noisy-SGD (DP-SGD without clipping), suggesting that the stochasticity (and not the clipping) is the cause of this implicit bias, \emph{even with additional isotropic Gaussian noise.}
We theoretically analyse the solutions obtained with continuous versions of Noisy-SGD for the Linear Least Square and Diagonal Linear Network settings, and reveal that the implicit bias is indeed amplified by the additional noise.
Thus, the performance issues of large-batch DP-SGD training are rooted in the same underlying principles as SGD, offering hope for potential improvements in large batch training strategies.
\end{abstract}
\label{sec:abstract}

\vspace{-0.5cm}
\section{Introduction}

\begin{figure}[b!]
    \centering
    \includegraphics[width=0.95\columnwidth]{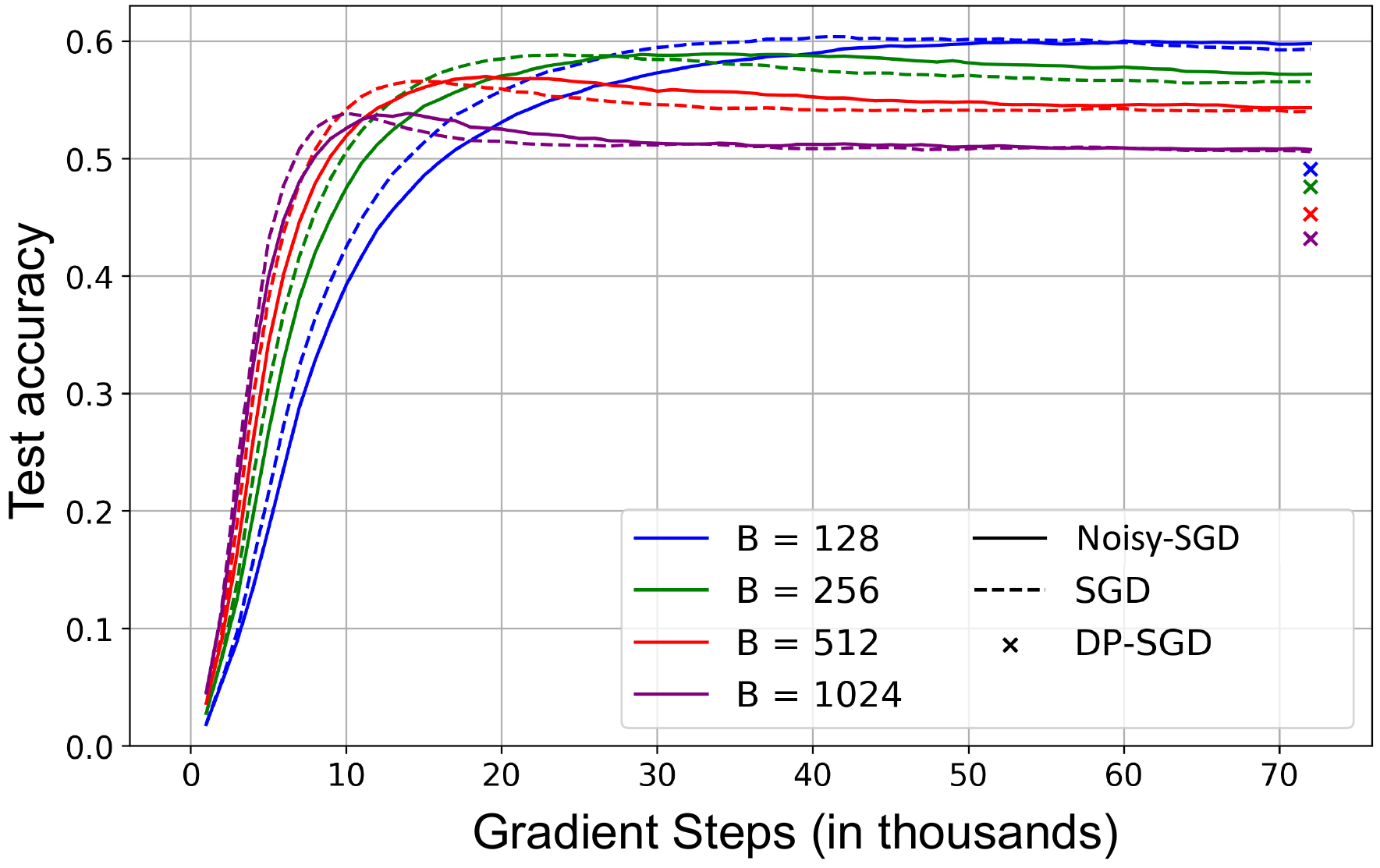}
    \vspace{-0.5em}
    \caption{
    Training from scratch on ImageNet for $S=72k$ steps, using a constant learning rate, with different batch sizes $B$. 
    The effective noise $\sigma/B$ is constant within DP-SGD and Noisy-SGD experiments. 
    The crosses for DP-SGD are obtained from \cite{sander2023tan}. 
    We observe a similar phenomenon for the non-clipped version (Noisy-SGD), \emph{i.e.}, small batches perform better that larger ones, suggesting that clipping is not solely responsible.
    Even with isotropic noise added with greater magnitude than the gradients, SGD's implicit bias persists: the natural noise structure in SGD is robust to Gaussian perturbations.
 \vspace{-1em}
 }
    \label{fig:Figure1}
\end{figure}
\vspace{-0.01in}

In Machine Learning, the Gradient Descent (GD) algorithm is used to minimize an empirical loss function by iteratively updating the model parameters in the direction opposite to the gradient. 
Its stochastic variant, Stochastic Gradient Descent (SGD) \citep{robbins1951, duflo1996algorithmes} uses a random subset of the training data at each step, known as a mini-batch, to estimate the true gradient. 
It enables the training of machine learning models on vast datasets or with extremely large models, where the computation of the full gradient becomes too computationally intensive.
Especially in Deep Learning, SGD or its variant has proven to be a crucial tool for training DNNs, delivering exceptional performance across various domains including computer vision~\citep{ResNets}, natural language processing~\citep{devlin2018bert, touvron2023llama}, and speech recognition \citep{amodei2016deep}.

SGD can achieve better performance than GD under a fixed compute budget (\emph{i.e.,} number of epochs), as well as at a fixed number of steps budget (\emph{i.e.,} number of updates) \citep{keskar2016large, SmallBS, masters2018revisiting}.
Consequently, not only does SGD offer significant computational resource savings, but its inherent stochastic nature introduces randomness into the algorithm which facilitates faster convergence and enhance generalization performance, by allowing the iterates to escape from unfavorable local minima.
On simpler, overparameterized model architectures, the distinctive noise structure of SGD is recognized as a factor responsible for yielding superior solutions compared to Gradient Descent, which is referred to as its ``implicit bias"  \citep{haochen2020shape}.

DNNs possess the ability to grasp the general statistical patterns and trends within their training data distribution ---such as grammar rules for language models--- but also to memorize specific and precise details about individual data points, including sensitive information like credit card numbers \citep{carlini2019secret, carlini2021extracting}.
This capability raises concerns regarding privacy, as accessing a trained model may lead to the exposure of training data, potentially jeopardizing privacy.
One possible solution to address this issue is Differential Privacy (DP) \citep{dwork2006calibrating}, which theoretically controls the amount of information learned from each individual training sample. 

Differentially Private Stochastic Gradient Descent (DP-SGD) \citep{abadi2016deep} offers a robust framework by providing strict DP guarantees for DNNs. 
The gradient of each training sample is clipped and Gaussian noise is subsequently introduced to the sum before the update.
Efforts to enhance the trade-off between privacy and utility in DP-SGD training have been associated with the use of exceedingly large batches \citep{yu2021large,DeepMindUnlocking, yu2023vip}.
A recent observation by \cite{sander2023tan} has introduced an intriguing challenge: under the same Gaussian noise, small batches perform considerably better than large batches, akin to the behavior of regular SGD.
It represents a bottleneck because large-batch training is essential for achieving robust privacy guarantees.


DP-SGD distinguishes itself from SGD by two factors: (1) the application of per-sample gradient clipping and (2) the addition of isotropic Gaussian noise to the sum per batch.
We first observe a persistent manifestation of the implicit bias associated with DP-SGD observed by \cite{sander2023tan} when we remove the gradient clipping component, as depicted in Figure~\ref{fig:Figure1}. 
Therefore, the implicit bias of SGD does persist even with Gaussian noise with greater magnitude than the gradients (see Figure~\ref{fig:grad_norm}).
However, the implicit bias in SGD is conventionally understood to stem from its inherent noise geometry \citep{haochen2020shape}.
In this context, we theoretically investigate the impact of changing the noise structure of SGD on its implicit bias, in the Linear Least Square and Diagonal Linear network (DLN) settings. 
Our precise contributions are the following:



\begin{enumerate}
    \item We show that the performance drop observed for large batch training in DP-SGD persists without clipping when training a DNN on ImageNet
    \item 
    For Linear Least Squares, we show how Noisy-SGD alters the limiting distributions: the different implicit bias compared to SGD can be controlled by the amount of additional noise
    \item For DLNs, we observe that Noisy-SGD can even exhibit an enhanced implicit bias in comparison to SGD. Using continuous modelings, we theoretically demonstrate that a favorable implicit bias indeed exists: It leads to the same solution as SGD, albeit with a distinct effective initialization. 
\end{enumerate}

It proves that the gradient's noise geometry is robust to Gaussian perturbations. 
Our work also suggests that enhancing the performance of large batch training with DP-SGD, which is crucial for achieving robust privacy-utility trade-off, can be accomplished by adopting strategies and techniques developed for managing large batches in non-private settings.


\label{sec:intro}

\section{Background and Related Work}

\subsection{Differential Privacy}

\paragraph{General Introduction} $\calM$ is a mechanism that takes as input a dataset $D$ and outputs a machine learning model $\theta \sim \calM(D)$.
\begin{definition}[Approximate Differential Privacy]
A randomized mechanism $\calM$ satisfies $(\varepsilon, \delta)$-DP~\citep{dwork2006calibrating} if, for any pair of datasets $D$ and $D'$ that differ by one sample and for all subset $R\subset \mathbf{Im}(\calM)$,
\begin{equation}
    \prob(\calM(D) \in R) \leq \prob(\calM(D') \in R) \exp( \varepsilon) + \delta.
\end{equation}
\end{definition}

Differential Privacy (DP) safeguards against the ability of any potential adversary to infer information about the dataset $\mathcal{D}$ once they've observed the output of the algorithm. 
In machine learning, this concept implies that if someone acquires the model's parameter $\bm{\theta}$ trained with a DP algorithm $\mathcal{M}$, then the training data is, by provable guarantees, difficult to reconstruct or infer~\citep{balle2022reconstructing, guo2022bounding,guo2023analyzing}.

\paragraph{DP-SGD}~\citep{chaudhuri2011differentially, abadi2016deep} is the most popular DP algorithm to train DNNs.
It first selects samples uniformly at random with probability $q=B/N$. 
For $C >0$, define the clipping function for any $X \in \mathbb{R}^d$ by $\mathrm{clip}_C(X) = C\cdot X/\Vert X \Vert$ if $\Vert X \Vert \geq C$ and $\mathrm{clip}_C(X) = X$ otherwise. DP-SGD clips the per sample gradients, aggregates them and adds Gaussian noise.
Given model parameters $\bm{\theta}_k$, DP-SGD defines the update
$\bm{\theta}_{k+1} = \bm{\theta}_k -\eta_k \widetilde{\mathbf{g}}_k$ where $\eta_k$ is the step size and $\widetilde{\mathbf{g}}_k$ is given by:
\begin{equation}
\label{eq:DP_SGD}
\widetilde{\mathbf{g}}_k := \frac{1}{B} \sum_{i \in \mathcal{B}_k} \text{clip}_C\left(\nabla_{\bm{\theta}} \ell_i(\bm{\theta}_{k})\right) + \gauss \left(0, C^2  \frac{\sigma^2}{B^2} \mathbf{I} \right).
\end{equation}
$\ell_i(\bm{\theta})$ represents the individual loss function computed for the sample $\mathbf{x}_i$. 
The privacy analysis of DP-SGD hinges on the combination of multiple steps. 
A notably robust analytical framework to account for $(\varepsilon, \delta)$ is founded on R\'{e}nyi differential privacy~\citep{mironov2017renyi}. The theoretical analysis of the convergence properties of DP-SGD has been studied for the convex, strongly convex, and nonconvex settings~\citep{bassily2014private, wang2017differentially, feldman2018privacy}.





\paragraph{Necessity of Large batches for DP-SGD} When it comes to training machine learning models with DP-SGD, there is necessarily a privacy-performance trade-off.
However, it can be improved by employing very large batch sizes~\citep{DPBert, li2021large, DeepMindUnlocking, yu2023vip}.

One practical underlying reason is that the magnitude of the effective noise introduced into the average clipped gradient is determined by $\sigma / B$ (\emph{cf.} Equation \eqref{eq:DP_SGD}).
To mitigate this noise, there are two potential approaches: reducing $\sigma$ or increasing $B$. 
While decreasing $\sigma$ might initially appear appealing, R\'{e}nyi differential privacy accounting techniques suggest that as $\sigma$ decreases too much, the privacy parameter $\varepsilon$ experiences an exponential increase~\citep{dwork2016concentrated, bun2016concentrated, mironov2019r, sander2023tan}.
Therefore, increasing the batch size emerges as a critical strategy in DP-SGD training, as it is the most effective means of reducing the effective noise and accelerating the convergence process.

While it is true that employing a larger batch size in DP-SGD results in reduced effective noise, it is important to note that DP training can experience a substantial drop in performance when using larger batch sizes. 
As demonstrated by \citet{sander2023tan}, when training an image classifier with DP-SGD on the ImageNet dataset, keeping the number of steps $S$ fixed  and maintaining a constant effective noise level of $\sigma/B$, the model's performance exhibited a notable decrease when they increased the batch size (refer to Figure~\ref{fig:Figure1}).

\subsection{Implicit Bias of SGD}\label{sec:prelim_implicit}

This section introduces the implicit bias of SGD through the use of Stochastic Differential Equations, and is largely based on \cite{pillaud2022}.
Let $X \in \mathbf{R}^{n\mathrm{x}d}$ and $Y \in \mathbf{R}^{n}$ be the training features and labels, matrices that represent $(x_i, y_i)_{1\leq i\leq n}$, the set of input-label training pairs. Let $\bm{\theta}\in \mathbb{R}^d$ the model's parameters and
$\Bar{X} \defeq X/\sqrt{n}$ the normalized features. 

\paragraph{General introduction}\label{sec:bias_intro}
We consider generic predictors $h$ and the square loss. 
The empirical risk is:

\vspace{-0.1in}
\begin{equation}
    R_n(\bm{\theta}) = \frac{1}{2n}\sum_{i=1}^n (h(\bm{\theta}, x_i)-y_i)^2
\end{equation}
\vspace{-0.1in}

At step $t$, the update with learning rate $\gamma$ is:

\vspace{-0.1in}
\begin{equation}\label{eq:SGD_model}
    \bm{\theta}_t =\bm{\theta}_{t-1} - \gamma \nabla_{\theta} R_n(\bm{\theta_{t-1}}) + \gamma \varepsilon_t(\bm{\theta}_{t-1})
\end{equation}
\vspace{-0.1in}

where $\varepsilon_t(\bm{\theta})$ is the noise term, that depends on the example(s) used to estimate the gradient. 
If only example with index $i_t$ is used:

\vspace{-0.15in}
\begin{equation}\label{eq:SGD_noise}
\varepsilon_t(\bm{\theta}) = \frac{1}{n}\sum_{i=1}^n r_j(\theta) \nabla_{\theta}h(\theta, x_j) - r_{i_t}(\theta) \nabla_{\theta}h(\theta, x_{i_t})
\end{equation}
\vspace{-0.1in}

With $r_i(\theta)=h(\theta, x_i)-y_i$. We refer to \cite{wojtowytsch2021stochastic} for additional references on the noise. It leads to a first notable characteristic of the gradient noise:
\begin{itemize}
    \item \textbf{The geometry}. SGD noise lies in a linear subspace of dimension at most $n$ spanned by the gradients: $\varepsilon(\bm{\theta_t}) \in \text{span} \{\nabla h(\bm{\theta}_t, x_1), ..., \nabla h(\bm{\theta}_t, x_n)\}$, which is a strict subspace of $\mathrm{R}^d$.
\end{itemize}

Training overparametrized  (\emph{i.e.} $d > n$) models with small batches can lead to better generalization performance compared to large batch training~\citep{keskar2016large, SmallBS, masters2018revisiting}. 
In this case, SGD is not only beneficial in terms of computational complexity, but also induces a bias that is beneficial to the performance. 
\emph{In the overparametrized case}, a second characteristic of the gradient noise is:  
\begin{itemize}
    \item \textbf{The scale}. The noise vanishes near optimal solutions: $\varepsilon(\bm{\theta^*})=0$.
\end{itemize}

The fact that SGD converges towards a particular interpolator is refered to as its implicit bias \citep{pesme2021implicit}. 
``Implicit" because no explicit regularization term is added; the regularization comes from the stochastic noise of estimating the gradient at each step using a mini-natch of samples ~\citep{zhang2017understanding}.

\paragraph{Stochastic Differential Equations}\label{sec:SDEs}
At constant step size, SGD is a homogeneous Markov chain \citep{meyn2012markov}.
Studying the continuous time counterparts of numerical optimization methods is a well-established field in applied mathematics, as it helps to study the limiting distribution of the iterates.
Due to its stochastic nature, SGD cannot be modeled as a deterministic flow.
A natural approach is to use stochastic differential equations (SDEs) \citep{Oksendal2003} to represent its dynamics:

\begin{equation}
    d\theta_t = b(t, \theta_t)dt + \sigma(t, \theta_t)dB_t
\end{equation}

where $B$ is a standard Brownian motion. 
The drift term $b$ corresponds to the negative gradient of the risk function, and the noise term represents the stochasticity of SGD, which must meet $\sigma_t\sigma_t^T=\gamma\mathbb{E}[\varepsilon_t\varepsilon_t^T|\theta_t]$ and $\sigma_t \in \text{span} \{\nabla h(\bm{\theta}_t, x_1), ..., \nabla h(\bm{\theta}_t, x_n)\}$ \citep{SDE_SGD}.

\paragraph{Linear Least Square.}We minimize in $\theta$:

\vspace{-0.15in}
\begin{equation}
    L(\theta) = \frac{1}{2n}\sum_{i=1}^n(\langle \theta, x_i \rangle - y_i)^2
\end{equation}
\vspace{-0.1in}

In the overparametrized case, even if there is an infinite number of interpolators, GD and SGD will both converge to $\theta^{LS}$, the solution with the smallest $l_2$ norm, leading to the same implicit bias  \citep{zhang2017understanding}. 

\paragraph{Diagonal Linear Networks (DLN)}\label{sec:prelim_DLN}
The implicit bias of SGD appears for more complex architectures.
For instance, \cite{chizat2020implicit} have studied a two-layer neural networks trained with the Logistic Loss, and \cite{pesme2021implicit} a sparse regression using a 2-layer DLN. 
In this work, we focus on the latter.
DLN corresponds to a toy neural network that has garnered significant interest in the research community \citep{vaškevičius2019implicit, woodworth20Kernel, haochen2020shape}.
This attention stems from the fact that it is an informative simplification of nonlinear models.
The forward pass is $\langle u \cirbd v, x \rangle$ where $\cirbd$ is a term by term multiplication, which can equivalently be written $\langle w_+^2 - w_-^2, x_i \rangle$, for $u, v, w_-, w_+ \in \mathbb{R}^d$.
The the goal is to minimize in $w=(w_+,w_-)$ the loss: 
\begin{equation*}\label{eq:DLN}
    L(w) \defeq \sum_{i=1}^n(\langle w_+^2 - w_-^2, x_i \rangle-y_i)^2 =: \sum_{i=1}^n(\langle \beta_, x_i \rangle-y_i)^2
\end{equation*}

It is similar to the linear least square problem but this time the optimization is (non convex) in $w$. Considering GD's gradient flow from the initialisation $w_{0, \pm} = \alpha$, \cite{woodworth20Kernel} have shown  that $\beta := w_+^2 - w_-^2$ follows a mirror flow \citep{BECK2003167} on the loss $L$ and with the hyperbolic entropy:

\vspace{-0.15in}
\begin{equation}\label{eq:hyperbolic_potential}
\phi_{\alpha}(\beta) = \frac{1}{4}\sum_{i=1}^d \beta_i \text{ arcsinh}(\frac{\beta_i}{2\alpha_i^2}) - \sqrt{\beta_i^2+4\alpha_i^4}
\end{equation}
\vspace{-0.1in}

as a potential. It means that $\beta = w_+^2 - w_-^2$ follows the dynamic $d\nabla\phi_\alpha(\beta_t) = -\nabla_\beta R(\beta)dt$.
It converges to the limit $ \beta_\infty^\alpha := \argmin_{\beta s.t. X\beta = Y} \phi_\alpha(\beta)$.

For large initialisations $\alpha$, $\phi_\alpha$ is close to the $l_2$-norm, which means that the recovered solution has a small $l_2$-norm. 
On the other hand, for smaller values of $\alpha$, the potential  aligns with the $l_1$-norm. 
As a result, the retrieved solution will inherently possess some sparsity.
In the case of a sparse regression, using small values for $\alpha$ thus results in an improved implicit bias.

Expanding upon the GD analysis, \cite{pesme2021implicit} introduced a continuous modeling approach for SGD. Their research demonstrated that, starting from the same initialization $\alpha$, the stochastic process converges with high probability towards $\beta_\infty^{\alpha_\infty}$:

\vspace{-0.15in}
\begin{equation}\label{eq:DLN_SGD_optim}    \beta_\infty^{\alpha_\infty} = \argmin_{\beta s.t. X\beta = Y} \phi_{\alpha_\infty}(\beta),
\end{equation}
\vspace{-0.1in}

which is similar to GD, but with a smaller effective $\alpha$:

\vspace{-0.15in}
\begin{equation}
    \alpha_t = \alpha \cirbd \exp \left(-2 \gamma \mathrm{diag}(\Bar{X}^T\Bar{X}) \int_{0}^{t}L(\beta_s) ds\right)
\end{equation}
\vspace{-0.1in}

Thus, SGD leads to sparser solutions than GD, which is a proof of the implicit bias of SGD for DLNs.


\label{sec:prelim}

\section{Implicit bias of noisy-SGD}\label{sec:theory}


\paragraph{Notations.}
$(B_{t, \pm})_t $ and $ (\Tilde{B}_{t, \pm})_t$ are standardized Brownian motions on $\mathbb{R}^n$ and $\mathbb{R}^d$ respectively.

\subsection{Noisy SGD}\label{sec:noisy_SGD}

DP-SGD (\emph{cf.} Equation~\ref{eq:DP_SGD}) differentiates itself from SGD through (1) the utilization of per-sample gradient clipping and (2) the incorporation of isotropic Gaussian noise into the batch-wise gradient sum.
To study the reason behind the implicit bias of DP-SGD observed in \cite{sander2023tan} (\emph{i.e.}, that small batches perform better than large batches at fixed effective noise $\sigma/B$), we examine if a similar phenomenon exists without clipping (Noisy-SGD), thus using the following noisy gradient update:

\begin{equation}
\label{eq:noisy_SGD}
\widetilde{\mathbf{g}}_k := \frac{1}{B} \sum_{i \in \mathcal{B}_k} \nabla_{\bm{\theta}} \ell_i(\bm{\theta}_{k}) + \gauss \left(0, \frac{\sigma^2}{B^2} \mathbf{I} \right)
\end{equation}

We demonstrate in Figure~\ref{fig:Figure1} the persistence of this small-batch training superiority even in the absence of clipping, and even when confronted with Gaussian noise that surpasses the magnitude of the gradients (as depicted in Figure~\ref{fig:grad_norm}; see Section~\ref{sec:noisy_SGD_ImageNet} for experimental details).
This piques our interest in delving into the theoretical aspects of Noisy-SGD for simple architectures in order to explore the extent to which the natural gradient geometry remains robust to Gaussian perturbations. 
The implications of studying Noisy-SGD as a proxy for DP-SGD are discussed in Section~\ref{sec:discussion}.

\begin{figure}[b!]
    \centering
    \includegraphics[width=\columnwidth]{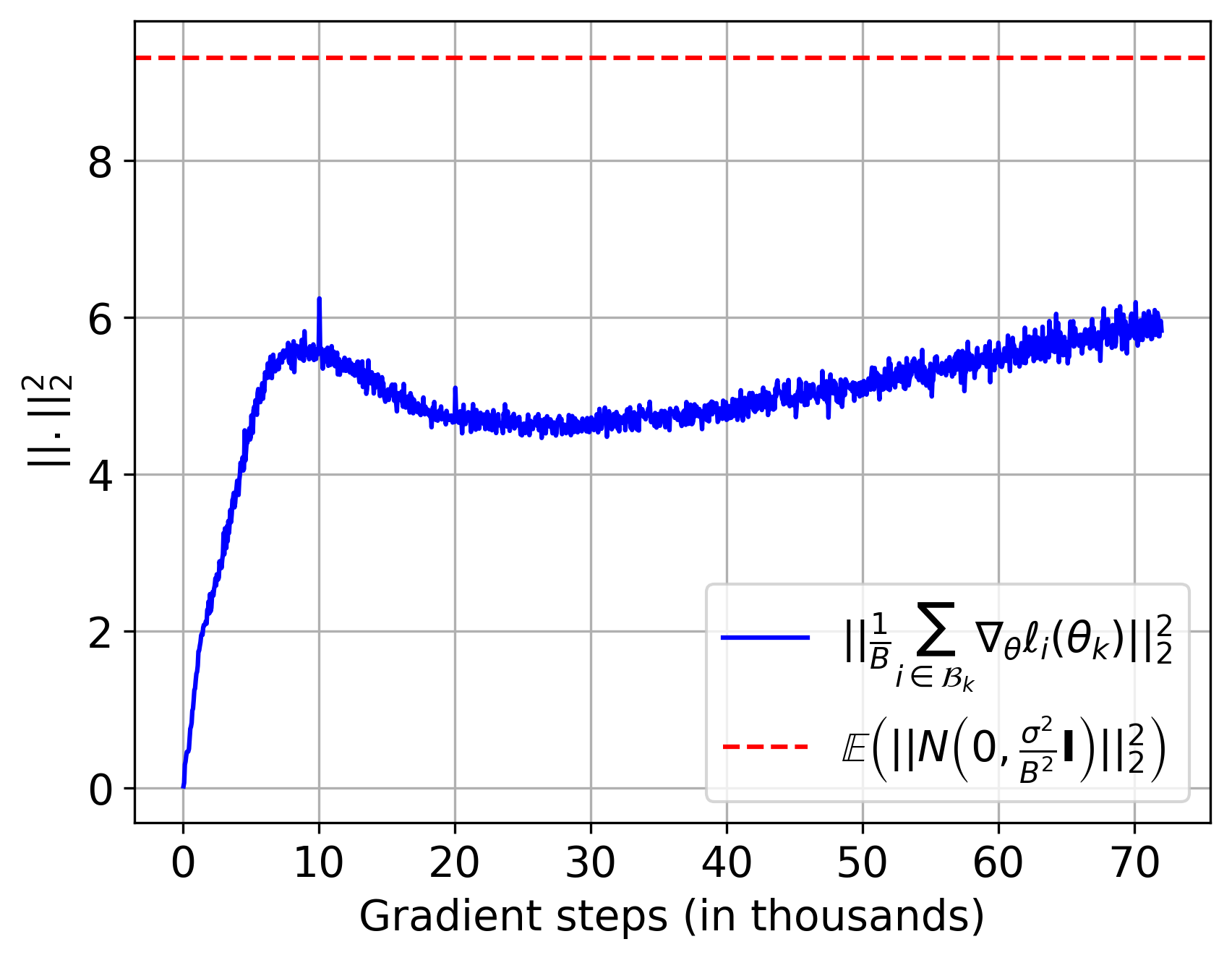}
    \vspace{-1em}
    \caption{Noisy-SGD on ImageNet. We compare the norm of the mini-batch gradient to the one of the Gaussian noise when training with Noisy-SGD on ImageNet, for $B=128$ and the same set-up as in Figure~\ref{fig:Figure1}. The noise magnitude is greater than the gradients. }
    
 \vspace{-1em}
    \label{fig:grad_norm}
\end{figure}

\begin{figure*}[t]
    \centering
    \includegraphics[width=0.95\textwidth]{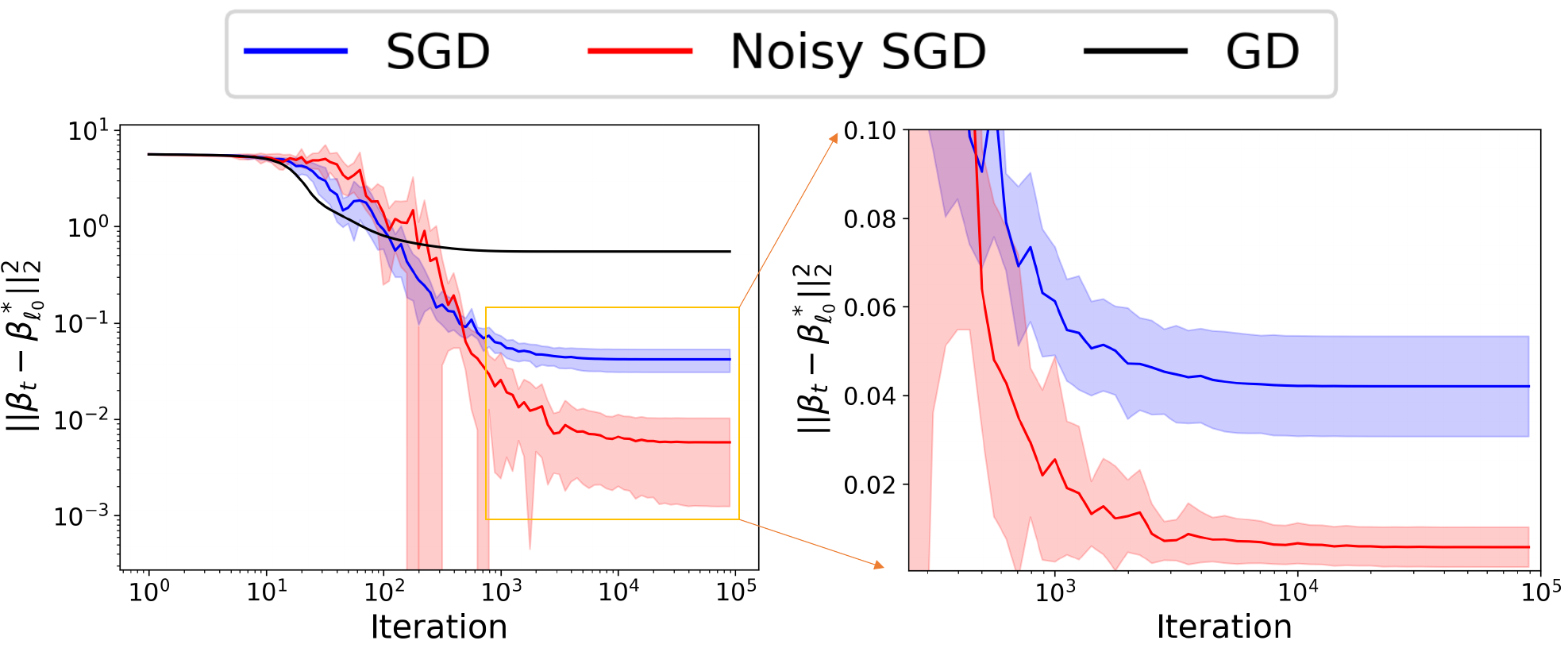}
    \vspace{-1em}
    \caption{
     Diagonal Linear Network: Implicit Bias of GD, SGD and Noisy SGD ($\sigma=0.5$ in Equation~\ref{eq:noisy_SGD_DLN}). 
    Shaded areas represent one standard deviation over 5 runs.
    (Left) Compared to GD with the same initialisation $\alpha=0.1$, SGD attain solutions closer to the sparse $\beta^*_{l_0}$, as expected from \cite{pesme2021implicit}. 
    Moreover, we observe that Noisy-SGD has a better implicit bias than SGD: the gradient noise structure is enhanced by perturbations.
    (Right) In absolute terms, Noisy-SGD does not even showcase more variance than SGD (near convergence).
 \vspace{-0.2in}
 }
    \label{fig:Figure3}
\end{figure*}

\subsection{Linear Least Square}\label{sec:Linear}

The simplest model to study is Linear Least Square.
We verify that noisy versions of SGD leads to similar solutions than SGD and GD.


\paragraph{Warm-up: Underparametrized setting ($n > d$)}The stochastic noise $\varepsilon(\theta)$ (\emph{cf.} Equation~\ref{eq:SGD_model}) does not vanish near optimal solutions. 
For fixed $\epsilon>0$, we consider the following SDE as the continuous approximation of SGD \citep{ali20a}:

\vspace{-0.15in}
\begin{equation}
    d\theta_t = -\Bar{X}^{T}(\Bar{X}\theta_t-Y)dt + \sqrt{\gamma}\epsilon\Bar{X}^TdW_t
\end{equation}
\vspace{-0.1in}

We notice that it respects the characteristics highlighted in Section~\ref{sec:SDEs}.
This Orsnten-Uhlenbeck process has a limit, which can be solved through its characteristic Lyapunov equation (see Appendix \ref{appendix:least}). 
With $X^\dagger$ the pseudo-inverse of $X$ and $\theta^ {LS}:=X^{\dagger} Y$:

\vspace{-0.1in}
\begin{equation}
    \theta_{\infty} \sim \mathcal{N}\left(\theta^ {LS}, \frac{\gamma \epsilon^2}{2} I_d \right)
\end{equation}
\vspace{-0.1in}

We now consider adding isotropic Gaussian noise on top of the gradient's structured noise of SGD:

\vspace{-0.15in}
\begin{equation}
    d\theta_t = -\Bar{X}^{T}(\Bar{X}\theta_t-Y)dt + \sqrt{\gamma}\epsilon\Bar{X}^TdW_t + \sigma d\Tilde{W_t}
\end{equation}
\vspace{-0.1in}

Solving the characteristic Lyapunov equation (see Appendix \ref{appendix:least}), the limiting distribution becomes:

\vspace{-0.1in}
\begin{equation}
\boxed{
    \theta_{\infty} \sim \mathcal{N}\left(\theta^ {LS}, \frac{\gamma \epsilon^2}{2} I_d + \frac{\sigma^2}{4}(\Bar{X}^{T}\Bar{X})^{-1}\right)
}
\end{equation}
\vspace{-0.1in}

We observe that adding isotropic noise changes the limiting distribution: its shape depends on the training data.
This simple example shows that geometry variation of the noise in the linear least square setting implies a controlled variation of the limiting process.

\paragraph{Over parametrized case ($d > n$):} As described in Section~\ref{sec:prelim_implicit}, GD and SGD will in this case both converge to $\theta^{LS}$ and therefore have the same implicit bias.
SGD's dynamic can be studied through the following SDE \citep{ali20a}, where this time the gradient noise vanishes near the interpolators:

\vspace{-0.15in}
\begin{equation}\label{eq:SDE_linear_overparametrized}
    d\theta_t = -\Bar{X}^{T}(\Bar{X}\theta_t-Y)dt + \sqrt{\gamma}\norm{X\theta_t - Y}_2\Bar{X}^T dB_t
\end{equation}
\vspace{-0.1in}

which also respects SGD's specificities. 
$\theta_t$ converges to $\theta^{LS}$ for $\gamma \leq 1/Tr(\Bar{X}^T\Bar{X})$. 
If we add a constant spherical Gaussian noise at each step, the iterates will not converge.
We thus look at the impact of adding a noise with a similar scale as the gradient noise:


\vspace{-0.15in}
\begin{equation*}\label{eq:NSDE_linear_overparametrized}
\begin{split}
    d\beta_t & = -\Bar{X}^{T}(\Bar{X}\beta_t-Y)dt + \sqrt{\gamma}\norm{X\beta_t - Y}_2\Bar{X}^T dB_t  \\
   & + \space \sqrt{\gamma}\norm{X\beta_t - Y}_2 \sigma d\Tilde{B}_t 
\end{split}
\end{equation*}
\vspace{-0.15in}

We show that we can control the difference between the SDE and the noisy SDE through $\eta_t = \norm{\theta_t - \beta_t}_2^2$:

\begin{theorem}
    If $\gamma \leq 1/Tr(\Bar{X}^T\Bar{X})$ then
    \begin{equation}
\boxed{
    \mathbf{E}(\eta_t)\leq \gamma d\sigma^2 \int_{0}^{t} L(\beta_t) ds
}
\end{equation}
\end{theorem}

\begin{proof}
    Apply Itô's formula to $\eta_t$, see Appendix \ref{appendix:least}
\end{proof}

Here, we have highlighted that the variation from the original SDE will be contingent on factors such as dimension, step size, noise magnitude, and convergence rate. 
As a result, the solution obtained from a noisy variant of SGD may closely resemble that of traditional SGD, thus leading to a similar implicit bias.
The extent of this similarity depends on the parameters used.

\subsection{Noisy-SGD training of DLNs}\label{sec:DLN}

For a 2-layer DLN sparse regression, SGD-induced noise steers the optimization dynamics towards advantageous regions that exhibit better generalization than GD (\emph{i.e.}, more sparse) \cite{pesme2021implicit}.
We investigate the impact of the addition of Gaussian noise.

\vspace{-0.1in}
\paragraph{Continuous model for Noisy-SGD}
Adding constant Gaussian noise at each gradient step as in Equation~\ref{eq:noisy_SGD} inevitably makes the process diverge for overparametrized DLNs.
However, adding Gaussian noise with comparable magnitude to that of the gradients still allows to study the impact of geometry deformation, while maintaining a vanishing noise, enabling convergence. 
\cite{pesme2021implicit} have shown that the SGD update can be written $w_{t+1,\pm} = - \gamma\nabla_{w_{t,\pm}} L(w_t) \pm \gamma [X^T\zeta_{i_t}(\beta_t)]\cirbd w_{t, \pm}$, where $\zeta_{i_t}$ is detailed in Appendix~\ref{appendix:DLN}. 
We thus consider the following noisy update:

\vspace{-0.15in}
\begin{equation}\label{eq:noisy_SGD_DLN}
\begin{split}
    w_{t+1,\pm} = &   - \gamma\nabla_{w_{t,\pm}} L(w_t)   \pm \gamma [X^T\zeta_{i_t}(\beta_t)]\cirbd w_{t, \pm} \\ & \pm \gamma \sigma_t Z_{t,\pm} \cirbd w_{t, \pm},
\end{split}
\end{equation}
\vspace{-0.15in}


with $Z_{t,\pm} \sim \mathcal{N}_d(0,1)$ and $\sigma_t \in \mathbb{R}^{\mathbb{R}}$. 
We consider here $\sigma_t = 2\sigma \sqrt{L(w_t)}$ for $\sigma\geq 0$, and show more general results in Appendix~\ref{appendix:DLN}. 
The corresponding SDE is:






\vspace{-0.15in}
\begin{equation}\label{eq:SDE_gn_DLN}
\begin{split}
    dw_{t,\pm} = & \mp [\Bar{X}^T r(w_t)]\cirbd w_{t,\pm}dt \\&+ 2\sqrt{\gamma L(w_t)}w_{t,\pm} \cirbd \Bar{X}^TdB_{t,\pm} \\& + 2\sigma\sqrt{\gamma L(w_t)}w_{t,\pm} \cirbd d\Tilde{B}_{t,\pm}
\end{split}
\end{equation}
\vspace{-0.1in}

where $r(w) = \Bar{X}(w^2_+-w^2_--\beta^\ast)$ and $\beta^*$ is an interpolator.
We notice that compared to the SDE of \cite{pesme2021implicit}, it only adds the term $2\sigma\sqrt{\gamma L(w_t)}w_{t,\pm} \cirbd d\Tilde{B}_{t,\pm}$.
See section~\ref{sec:discussion} for a discussion on the impact of modeling with a decreasing noise, and appendix \ref{appendix:DLN} for detailed computation of the link between equations~\eqref{eq:noisy_SGD_DLN} and \eqref{eq:SDE_gn_DLN}.

We now demonstrate that despite the Gaussian perturbation, $\beta$ follows a stochastic mirror flow similar to the one in \cite{pesme2021implicit}, detailed in Section~\ref{sec:prelim_DLN}:
\begin{proposition}
    Let $(w_{t,\pm})_{t \geq 0}$ be defined as in equation \eqref{eq:SDE_gn_DLN} from initialisation $\alpha$. Then $(\beta_t=w_{t,+}^2-w_{t,-}^2)_{t\geq 0}$ follows a stochastic mirror flow defined by:
    \begin{equation}\label{eq:noisy}
\begin{split}
d\nabla\phi_{\alpha_t^\beta}(\beta_t) = & - \nabla L(\beta_t)dt + \sqrt{\gamma L(\beta_t)}\Bar{X}^TdB_t \\ & + \sqrt{\gamma L(\beta_t)}\sigma d\Tilde{B}_t
\end{split}
\end{equation}
where $\Tilde{B}_{s} = (\Tilde{B}_{s, +} + \Tilde{B}_{s, +})/{2}$, $B_{s} = (B_{s, +} + B_{s, +})/{2}$,  
\begin{equation}\label{eq:alpha_noisy}
\begin{split}
\alpha^{\beta}_t = & \alpha \exp{(-2\gamma\sigma^2\int_0^t L(\beta_s)ds)} \\& \cirbd\exp{(-2\gamma \mathrm{diag}(\Bar{X}^T\Bar{X}) \int_0^t L(\beta_s)ds)}
\end{split}
\end{equation}
\end{proposition}

\vspace{-0.15in}
\begin{proof}
    See appendix \ref{appendix:DLN}.
\end{proof}







The continuous version of SGD analysed by \cite{pesme2021implicit} had led to the following mirror flow:

\vspace{-0.15in}
\begin{equation}\label{eq:not_noisy}
    d\nabla\phi_{\alpha_t}(\beta_t) = - \nabla L(\beta_t)dt + \sqrt{\gamma L(\beta_t)}\Bar{X}^TdB_t
\end{equation}
\vspace{-0.15in}

With  $\alpha_t = \alpha \exp \left(-2 \gamma \mathrm{diag}(\Bar{X}^T\Bar{X}) \int_{0}^{t}L(\beta_s) ds\right)$.
Our resulting SDE for Noisy-SGD (\emph{cf.} Equation~\ref{eq:noisy}) differs by two aspects: the additional spherical noise $\sqrt{\gamma L(\beta_t)}\sigma d\Tilde{B}_t$ and a smaller effective value for $\alpha$.

\paragraph{Solution.} If $\beta$ follows the original mirror flow, than the process converges with high probability towards $\argmin_{\theta \text{ s.t. } X\theta=Y} \phi_{\alpha_{\infty}}$.
This formulation is possible only because the KKT conditions are respected by the limit vector, as the updates remain confined to $\text{span}(x_1,...,x_n)$.
However, this is no longer the case for Noisy-SGD. We still show a comparable outcome.

\begin{theorem}\label{thm:convergence_DLN}
    Let $(w_t)_{t\geq 0}$ follow the SDE \eqref{eq:SDE_gn_DLN} from initialisation $\alpha$. Then for any $p$ close to $0$ there is a constant $C$ s.t. for any step size $\gamma \leq C$, with probability at least $1-p$, $\int_{0}^\infty \beta_s ds$ converges and $\beta_t$ converges with high probability to an interpolator $\beta_\infty$.
\end{theorem}

\vspace{-0.15in}
\begin{proof}
    We give here an idea of the proof which is along the lines of the proof found in \cite{pesme2021implicit}. 
We first note that by arbitrarily changing $\Bar{X}^T$ in SDE \eqref{eq:not_noisy}, the process still converges. In our case: 

\vspace{-0.15in}
\begin{equation*}\label{eq:SDE_loss}
    dw_{t,\pm} = - \nabla_{w_\pm} L(w_t)dt + 2\sqrt{\gamma L(w_t)}w_{t,\pm} \cirbd [AdB_{t,\pm}] 
\end{equation*}
\vspace{-0.1in}

where this time $B_t$ is valued in $\mathbb{R}^{d+n}$ and
$A\defeq\left(\Bar{X}|\sigma I_d\right)$. The whole proof is given in Appendix~\ref{appendix:DLN}
\end{proof}
\vspace{-0.1in}

However, because the corresponding mirror flow updates no longer lie in $\text{span}(x_1,...,x_n)$ (\emph{cf.} Equation~\ref{eq:noisy}),  the KKT conditions of  the minimization problem are no longer verified by $\beta_\infty$, and 

\begin{equation}
    \beta_\infty \neq \argmin_{\beta \text{ s.t. } X\beta=Y} \phi_{\alpha^\beta_{\infty}}(\beta)
\end{equation}

Nevertheless, we show that $\beta_\infty$ satisfy the KKT conditions for a small perturbation of that problem:
\begin{proposition}\label{prop:min_problem_DLN}
Under the assumptions of Theorem~\ref{thm:convergence_DLN}, on the set of convergence of $\beta_t$, there exists $r_{\infty}$ s.t.:
  \begin{equation}
\boxed{
    \beta_{\infty} = \argmin_{X\beta=Y} \phi_{\alpha_{\infty}}(\beta) - \langle r_\infty, \beta\rangle
}
\end{equation}

\end{proposition}

\begin{proof}
Let us consider $P$ the orthogonal projection over span$(x_1, ..., x_n)$. Integrating \eqref{eq:noisy} and decomposing the right hand side as a term $v\in$span$(x_1, ..., x_n)$ and a term in the orthogonal, we get:

\vspace{-0.15in}
\begin{equation}
    \nabla\phi_{\alpha_{\infty}}(\beta_{\infty}) =  v + \int_{0}^\infty \sigma \sqrt{\gamma L(\beta_{s})} (1-P) dBs
\end{equation}
\vspace{-0.1in}

Let us define $r_{\infty} = \int_{0}^\infty \sigma \sqrt{\gamma L(\beta_{s})} (1-P) dBs$, and $\Tilde{\beta}_0=\nabla\phi_{\alpha_{\infty}}^{-1}(r_{\infty})$, which is well defined by strong convexity of $\phi_{\alpha_\infty}$ (Refer to Appendix~\ref{appendix:DLN}). We  have:

\vspace{-0.15in}
\begin{equation}
    \nabla\phi_{\alpha_{\infty}}(\beta_{\infty})-\nabla\phi_{\alpha_{\infty}}(\Tilde{\beta}_0) \in \text{span}(x_1, ..., x_n)
\end{equation}
\vspace{-0.1in}

Thus, by defining $D$ the following Bregman divergence:
  \begin{equation*}
    D_{\phi_{\alpha_{\infty}}}(\beta, \Tilde{\beta}_0) = \phi_{\alpha_{\infty}}(\beta)-\phi_{\alpha_{\infty}}(\Tilde{\beta}_0) - \langle \nabla\phi_{\alpha_{\infty}}(\Tilde{\beta}_0), \beta- \Tilde{\beta}_0\rangle
\end{equation*}
\vspace{-0.1in}

The gradient taken at $\beta_\infty$ lies in $\text{span}(x_1, ..., x_n)$, insuring that the KKT conditions of the following problem are verified and thus:

\vspace{-0.15in}
\begin{equation}\label{eq:divergence}
\boxed{
    \beta_{\infty} = \argmin_{X\beta=Y}D_{\phi_{\alpha_{\infty}}}(\beta, \Tilde{\beta}_0)
}
\end{equation}
\vspace{-0.1in}
\end{proof}
\vspace{-0.1in}


Therefore, Noisy-SGD gives the same solution as GD but on a different potential $\phi_{\alpha_{\infty}}$ and from an effective initialization $ \Tilde{\beta}_0$, while SGD only changes $\alpha$.
For Noisy-SGD, we note that $\alpha_{\infty}$ decreases with $\sigma$ (eq.~\ref{eq:alpha_noisy}): the more noise is added, the smaller is the effective $\alpha$.
If it did not change the effective initialization too, it would directly imply ``better" implicit bias.

\vspace{-0.1in}
\paragraph{Implicit Bias}We have observed that $\beta_\infty$ is the solution to a perturbed minimization problem. 
We show now that the distance between $\beta_\infty$ and the optimizer of $\phi_{\alpha_\infty}$ can be controlled by the perturbation $\sigma$:

\begin{proposition}\label{eq:impact_eff_init}
    Under the assumptions of theorem \ref{thm:convergence_DLN} and on the high probability set of convergence of $(\beta_t)$ let $\beta^\ast_{\alpha_{\infty}} := \argmin_{X\beta = Y} \phi_{\alpha_\infty}(\beta)$ and $r_{\infty} := \sigma \int_{0}^\infty \sqrt{\gamma L(\beta_{s})} (1-P) dBs$. Since $\alpha_\infty > 0$, $\phi_{\alpha_\infty}$ is $\mu$ strongly convex for $\mu > 0$. Then we have:
    \begin{equation}
        \frac{1}{\mu}\Vert r_\infty \Vert_2 \geq \Vert \beta^\ast_{\alpha_{\infty}} - \beta_\infty \Vert_2
    \end{equation}
\end{proposition}

\vspace{-0.15in}
\begin{proof}
See Appendix~\ref{appendix:DLN}.
\end{proof}
\vspace{-0.1in}

We observe that although increasing $\sigma$ implies a lower $\alpha_\infty$ than SGD, which should yield to sparser solutions, it also implies a new effective initialization $\Tilde{\beta}_0$ that deviates the solution from the sparse $\beta^\ast_{\alpha_{\infty}}$.
However, in Figure~\ref{fig:Figure3}, we show that Noisy-SGD still leads to sparser solutions than SGD.
This might be attributed to the fact that the gain from a smaller $\alpha$ is in this case more important than the perturbation controlled by $\frac{1}{\mu}\Vert r_\infty \Vert \propto \sigma$, as shown in Figure~\ref{fig:Figure4}.
We give more details in the experiment section~\ref{sec:DLN_exp}.

Proofs are derived in Appendix~\ref{sec:appendix_DLN_proofs} with more general noise forms, with stability results for the implicit bias.

\begin{figure}[b!]
    \centering
    \includegraphics[width=\columnwidth]{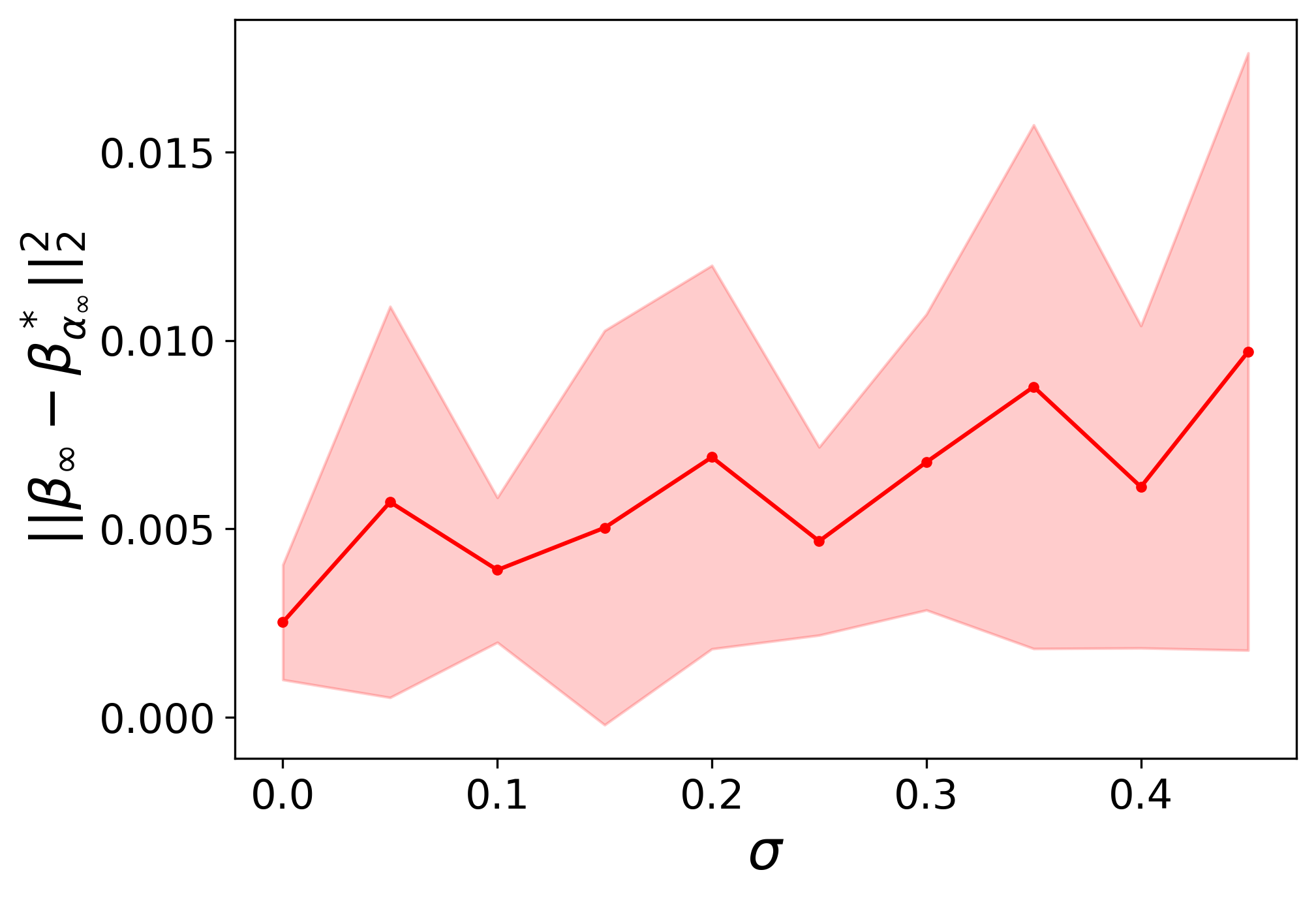}
    \vspace{-1em}
    \caption{
     DLN: Distance between $\beta^*_{\alpha_\infty}$, the solution that minimizes $\phi_{\alpha_\infty}$ ---obtained by GD from $\alpha_\infty$--- and the one obtained by Noisy-SGD (see Proposition~\ref{eq:impact_eff_init}).
    Shaded areas represent one standard deviation over 10 runs.
    For small $\sigma$, the distance is smaller than the distance between the solutions of SGD and the sparse solution $\beta_{l_0}$ (see Figure~\ref{fig:Figure3}), explaining why the implicit bias persists and can be enhanced by Gaussian noise.
 \vspace{-1em}
 }
    \label{fig:Figure4}
\end{figure}

\label{sec:method}

\section{Experiments}

\subsection{Noisy SGD on ImageNet}\label{sec:noisy_SGD_ImageNet}

\paragraph{Dataset and architecture} To compare Noisy-SGD to DP-SGD, we adopt the exact same set-up as \cite{sander2023tan}.
We train a Normalizer-Free ResNets (NF-ResNets)~\citep{NFnet} with $d=25$M parameters on the ImageNet-1K dataset~\citep{Imagenet,Imagenet2} with blurred faces, which contains $1.2$ million images partitioned into 1000 categories. 
We use the timm \citep{rw2019timm} library based on Pytorch \citep{paszke2019pytorch}.

\paragraph{Optimization}For SGD and Noisy-SGD, we use a constant learning rate of $0.5$ and no momentum. 
We perform an exponential moving average of the weights \citep{tan2019efficientnet} with a decay parameter of $0.999$ (similar to \cite{sander2023tan}).
We set the effective noise $\sigma/B$ constant to $6\times10^{-4}$ for all Noisy-SGD experiments of Figure~\ref{fig:Figure1}, which is 4 times greater than the noise used by \cite{sander2023tan} for ImageNet.

\paragraph{Implicit bias and Noise level}In Figure~\ref{fig:grad_norm}, we show that the additional Gaussian noise has a bigger $l_2$ norm than the gradients throughout training, for the same set-up as in Figure~\ref{fig:Figure1}.
It is an important observation because it shows that the additional Gaussian noise is not negligible, and thus that a similar implicit bias to the one of SGD persists even under strong Gaussian perturbation.
Moreover, we notice that the training trajectories of SGD and Noisy-SGD do differ in Figure~\ref{fig:Figure1}, especially at the beginning. 
However, the gradient norm $||\frac{1}{B} \sum_{i \in \mathcal{B}_k} \nabla_{\bm{\theta}} \ell_i(\bm{\theta}_{k})||$ increases during training, while the corresponding quantity in DP-SGD is bounded by 1 (\emph{cf.} Equation~\ref{eq:DP_SGD} when $C=1$), which makes the actual signal-to-noise ratio greater in Noisy-SGD compared to DP-SGD.

\subsection{Diagonal Linear Network}\label{sec:DLN_exp}

We adopt the same set-up as \cite{pesme2021implicit} for sparse regression.
We select parameters n = 40 and d = 100, and then create a sparse model $\beta^*_{l_0}$ with the constraint that its $l_0$ norm is equal to 5. 
We generate the features $x_i$ from a normal distribution with mean 0 and identity covariance matrix $N(0, I)$, and compute the labels as $y_i = x_i^T \beta^*_{l_0}$. 
We always use the same step size of $\gamma = 1/(1.3||\Bar{X}\Bar{X}^T||_2)$. Notice that $||\beta_t - \beta^*_{l_0}||^2_2$ is the validation loss in the experiments.

\paragraph{Noisy-SGD's improved implicit Bias}In Figure~\ref{fig:Figure3}, we show that noisy-SGD produces sparser solutions than the one obtaimed with SGD, as it is closer to the sparse interpolator $\beta_{l_0}$.
This effect is primarily due to the impact of having a smaller effective value $\alpha$ which outweighs the impact of the perturbation governed by the new effective initialization $\Tilde{\beta_0}$: $\frac{1}{\mu}\Vert r_\infty \Vert \propto \sigma$ (see Equation~\ref{eq:impact_eff_init}).
In Appendix~\ref{sec:appendix_additional_exps}, we also illustrate scenarios where the parameter $\sigma$ becomes more prominent in comparison to $\alpha$, and for which Noisy-SGD still induce a stronger implicit bias.

\paragraph{Impact of $\sigma$} In Figure~\ref{fig:Figure4}, we run Noisy-SGD for different values of $\sigma$, $10$ times for each value, and show the averages and standard deviations. 
As expected from Proposition~\ref{eq:impact_eff_init}, the distance between $\beta^*_{\alpha_{\infty}}$, the actual minimizer under constraints of $\phi_{\alpha_\infty}$ and $\beta_\infty$, the solution obtained from Noisy-SGD, increases with $\sigma$. 
We observe that the order of magnitude of this increase, which was hidden inside the constants of the Proposition, is reasonable when compared to the distance between the solution of SGD and the sparse interpolator, as depicted in Figure~\ref{fig:Figure3}.
It explains why using $\sigma=0.5$ still helps when starting from $\alpha=0.1$.

\section{Conclusion}

The challenge of achieving strong performance in large-batch DP-SGD training remains a critical obstacle in balancing privacy and utility. 
This performance bottleneck in DP-SGD is poised to become more pronounced over time as the theoretical solution to further bridging the performance gap between private and non-private training lies in the expansion of both training set size and batch size \citep{sander2023tan}.

We have observed that Noisy-SGD exhibits a performance decline with the batch size, indicating that the gradient noise in SGD still plays a pivotal role, even with Gaussian perturbation. 
It puts in perspective a key argument supporting that the inherent geometry of SGD's noise is responsible for its bias. 
For DLNs, a seemingly minor addition of Gaussian noise disrupts a crucial KKT condition, necessitating a complex mathematical alternative. 
We've derived proofs for broader noise in Appendix~\ref{sec:appendix_DLN_proofs} and offer stability results that extend the amplification of the implicit bias.
Overall, our study underscores two critical points:

\begin{itemize}
    \item The specific gradient noise geometry in SGD is robust to Gaussian perturbations, and different geometries can lead to different implicit biases
    \item Further work could try to leverage methods developed to improve large-batch training in non-private settings (\emph{e.g.}, employing the LARS optimizer \citep{you2017large} for convolutional DNNs) to enhance the performance of DP-SGD, thus advancing the trade-off between privacy and utility. 
\end{itemize}

\vspace{-0.5cm}
\section{Discussion}\label{sec:discussion}


Our decision to study Noisy-SGD as a substitute for DP-SGD is primarily based on the observations depicted in figure~\ref{fig:Figure1}. 
Notably, a similar implicit bias of DP-SGD is discernible even in the absence of clipping.
We delve here into the implications of using Noisy-SGD as a stand-in DP-SGD.
This includes considerations from both optimization and privacy perspectives.

\paragraph{An optimization perspective.} 
Noisy-SGD only aligns with DP-SGD if the gradients are bounded, an assumption that is commonly done in convergence studies~\citep{bassily2014private}.
This is not the case for classical neural networks, nor for the DLN case that we studied. 
Without any assumption, the clipping operation can bias the (expected) direction of the update.
However, assuming symmetricity of the gradients, which is a reasonable assumption~\citep{chen2020understanding}, the drift maintains its direction after clipping. 

If we take clipping into account for DLNs while assuming symetricity, the noise term can still be approximated by $\sigma_t \circ w_t dB_t$, where $\sigma_t$ now depends on $w_t$.
Indeed, if we denote $y_{t} :=\langle \beta - \beta^\ast,x_{i_t}\rangle x_{i_t}$, we get $ w_{t+1} = w_t - \gamma \mathbb{E}[\text{clip}(y_{t}\circ w_t)] - \gamma(\text{clip}( y_{t} \circ w_t) - \mathbb{E}[\text{clip}(y_{t}\circ w_t)])$. The mean zero term rewrites as $\alpha(y_{t},w_t) y_{t} \circ w_t - \mathbb{E}[\alpha(y_{t},w_t) y_{t} \circ w_t]$, with first order of the covariance$(1/n) \text{diag}(w_t) \alpha^2 \mathbb{E}[y_{t} y_{t}^T ]\text{diag}(w_t)$. 
Our extension in appendix~\ref{sec:appendix_DLN_proofs} is a first approximation. 

\paragraph{A privacy perspective.}We deliberately focus on the sole effect of noise addition and its impact on the optimization procedure. 
Nevertheless, we do theoretically show how the noise level impacts our results, with additional experiments presented in Appendix~\ref{sec:appendix_additional_exps}, and general forms in appendix~\ref{sec:appendix_DLN_proofs}.
However, vanishing noise would always lead to exploding privacy guarantees. 
We use this modelisation as a first approximation, as it changes the structure of the noise similarly to DP.
Moreover, it is necessary for convergence as the learning is constant; studying a decreasing learning rate with fixed noise could have been an alternative, but we should have taken a different angle (unknown) to study the bias.
In a similar vein, one could consider the clipping component more comprehensively. 
One potential approach could be to proportionally decrease the clipping value in tandem with the noise magnitude. 
This method would not only uphold the privacy guarantees but also ensure convergence

\section*{Acknowledgements}
Our sincere thanks to Pierre Stock and Alexandre Sablayrolles for their initial guidance, Scott Pesme for his critical insights on the mathematical derivations of the DLN part, and Ilya Mironov for his constructive feedback on our draft.

\label{sec:experiments}

\bibliography{biblio}
\bibliographystyle{unsrtnat}

\newpage

%
%




%

%


\appendix
\onecolumn
\aistatstitle{Supplementary Materials}

\section{Linear Least Square}\label{appendix:least}

We detail the proofs of the Propositions presented in Section~\ref{sec:Linear} focusing on the continuous analysis of the implicit bias of SGD and Noisy-SGD in the Linear Least Square Setting. 
We start with the under parametrized case in Section~\ref{sec:appendix_least_under} where the number of parameters $d$ is smaller than the number of training data points $n$. 
Subsequently, we delve into the over parameterized case in  Section~\ref{sec:appendix_least_over}.

\subsection{Under Parameterized Case}\label{sec:appendix_least_under}

\paragraph{SDE for SGD} In situations where the model is underparameterized, the gradient noise encompasses the entire space of $\mathbb{R}^d$ and there is a specific value of $\epsilon>0$ such that $||X\theta_t - Y||$ remains greater than $\epsilon$ for all time steps $t$, as documented in \cite{ali20a}.
We thus consider the following SDE as the continuous approximation of SGD in the linear least square under parametrized setting.
In this framework, the Euler discretization with a step size of $\gamma$ corresponds to the implementation of SGD, incorporating the noise modelization previously described:

\begin{equation}
    d\theta_t = -\Bar{X}^{T}(\Bar{X}\theta_t-Y)dt + \sqrt{\gamma}\epsilon\Bar{X}^TdW_t
\end{equation}

This process is a Orsnten-Uhlenbeck process that can be written as:

\begin{equation}
    d\theta_t = \beta(\mu - \theta_t)dt + \Sigma dW_t
\end{equation}

with $\mu = (X^{T}X)^{-1}X^TY =X^{\dagger} Y := \theta^{LS}$ for $X^\dagger$ the pseudo-inverse of $X$,  $\beta = \Bar{X}^{T}\Bar{X}$ and $\Sigma=\sqrt{\gamma}\epsilon\Bar{X}^T$. 

The stationary distribution of this temporally homogeneous Markov Chain is:

\begin{equation}
    \theta_{\infty}  = \mathcal{N}(\theta^{LS}, w)
\end{equation}

for $w$ that verifies the following Lyapunov equation:

\begin{equation}
    \beta w + w\beta^T = 2D = \gamma\epsilon^2\beta
\end{equation}

For $D = \Sigma^T\Sigma/2 = \frac{\gamma\epsilon^{2}}{2}\beta$. The solution of this Lyapunov equation can be directly solved as:

\begin{equation}
    w = \int_{0}^{\infty} e^{-t\beta}(-2D)e^{-t\beta} \,dt = \int_{0}^{\infty} e^{-t\beta}(-\gamma\epsilon^2\beta)e^{-t\beta} \,dt = -\gamma\epsilon^2\beta \int_{0}^{\infty} e^{-2t\beta} \,dt = \frac{\gamma \epsilon^2}{2} I_d
\end{equation}

As $\int_{0}^{\infty} e^{-2t\beta} \,dt=\beta^{-1}/2$.
The commutativity results from the fact that all matrices are sums of powers of $\beta$ within the integrals.

\paragraph{Noisy-SDE}Let us now consider adding Gaussian noise to the modelisation of the natural SGD noise:

\begin{equation}
    d\theta_t = -\Bar{X}^{T}(\Bar{X}\theta_t-Y)dt + \sqrt{\gamma}\epsilon\Bar{X}^TdW_t + \sigma d\Tilde{W_t}
\end{equation}

we get the same equations for $\Sigma = (\sqrt{\gamma}\epsilon\Bar{X}^T | \sigma I_d) \in \mathbf{R}^{d\times(n+d)}$ and $W = (W^T|\Tilde{W}^T)^T$. 

So $D = \frac{\Sigma\Sigma^T}{2} = \frac{\frac{\gamma}{n}\epsilon^2X^TX + \sigma^2I_d}{2} = \frac{\gamma\epsilon^2}{2}\beta + \frac{\sigma^2}{2} I_d$, and in this case:

\begin{equation}
    w = \int_{0}^{\infty} e^{-t\beta}(-\gamma\epsilon^2\beta)e^{-t\beta} \,dt - \frac{\sigma^2}{2}(-\frac{1}{2}\beta^{-1})= \frac{\gamma \epsilon^2}{2} I_d + \frac{\sigma^2}{4}\beta^{-1}
\end{equation}

With $\beta^{-1} = (\Bar{X}^T\Bar{X})^{-1}$. Finally:

\vspace{-0.1in}
\begin{equation}
\boxed{
    \theta_{\infty} \sim \mathcal{N}\left(\theta^ {LS}, \frac{\gamma \epsilon^2}{2} I_d + \frac{\sigma^2}{4}(\Bar{X}^{T}\Bar{X})^{-1}\right)
}
\end{equation}
\vspace{-0.1in}

In this case, we observe that adding spherical noise adds a dependence on the data distribution to the variance term.

\subsection{Over Parameterized Case}\label{sec:appendix_least_over}

In the overparametrised regime, the noise vanishes at every global optimum $\theta^\star$, and is degenerate in the directions of $\text{Ker}(X)$, giving the following continuous approximation \citep{pillaud2022, ali20a}:

\vspace{-0.15in}
\begin{equation}
    d\theta_t = -\Bar{X}^{T}(\Bar{X}\theta_t-Y)dt + \sqrt{\gamma}\norm{X\theta_t - Y}_2\Bar{X}^T dB_t
\end{equation}
\vspace{-0.1in}

In this case, $\theta_t$ converges almost surely to $\theta^{LS}$ for $\gamma >0$ and $\theta_0=0$. Let us now consider the following noisy version of SGD, with an additional noise that has the same magnitude as the gradient noise:

\vspace{-0.15in}
\begin{equation*}
\begin{split}
    d\beta_t & = -\Bar{X}^{T}(\Bar{X}\beta_t-Y)dt + \sqrt{\gamma}\norm{X\beta_t - Y}_2\Bar{X}^T dB_t  \\
   & + \space \sqrt{\gamma}\norm{X\beta_t - Y}_2 \sigma d\Tilde{B}_t 
\end{split}
\end{equation*}
\vspace{-0.15in}

We show that we can control the difference between the iterates of the first SDE and its noisy version: $\eta_t = \norm{\theta_t - \beta_t}_2^2$. More precisely,

    If $\gamma \leq 2/Tr(\Bar{X}^T\Bar{X})$ then
    \begin{equation*}
\boxed{
    \mathbf{E}(\eta_t)\leq \gamma d\sigma^2 \int_{0}^{t} L(\beta_t) ds
}
\end{equation*}


\begin{proof}

We can write $d\beta_t = -\Bar{X}^{T}(\Bar{X}\beta_t-Y)dt + \sqrt{\gamma}\norm{X\beta_t - Y}_2(\Bar{X}^T|\sigma I_d)dB'_t$ for $B' = (\frac{B}{\Tilde{B}})\in\mathbb{R}^{n+d}$. 
Similarly to $\theta$, for $\beta_0=0$ and $\gamma>0$, $\beta$ converges almost surely. 
So $\theta-\beta$ converges and finally $\eta$ converges.
Applying Itô's formula, we get:

\begin{equation*}
\begin{split}
    d\eta_t & = [-2\norm{\Bar{X}(\theta_t-\beta_t)}^2 \\ & + \gamma Tr(\Bar{X}^T\Bar{X})(\sqrt{R_n(\theta_t)}-\sqrt{R_n(\beta_t)})^2 \\
    & + d\sigma^2\gamma R_n(\beta_t)]dt + ...dW_t + ... d\Tilde{W}_t
\end{split}
\end{equation*}

Straightforwardly using the triangular inequality, we have:

\begin{equation}\label{eq:triangular_ineq}
    (\sqrt{R_n(\theta_t)} - \sqrt{R_n(\beta_t)})^2 \leq ||\Bar{X}(\theta_t-\beta_t)||^2_2 = ||\Bar{X}\theta_t - Y - (\Bar{X}\beta_t - Y)||^2_2 \leq (\sqrt{R_n(\theta_t)} + \sqrt{R_n(\beta_t)})^2
\end{equation}

Looking now at the expected value, the Brownian motions disappear and injecting Inequality~\ref{eq:triangular_ineq}:

\begin{equation*}
\begin{split}
    \mathbb{E}(\eta_t) & =  \gamma d\sigma^2 \int_{0}^{t} R_n(\beta_s)ds - 2\int_{0}^{t}\norm{\Bar{X}(\theta_s-\beta_s)}_2^2ds + \gamma Tr(\Bar{X}^T\Bar{X}) \int_{0}^{t}(\sqrt{R_n(\theta_s)}-\sqrt{R_n(\beta_s)})^2 \\ & \leq \gamma d\sigma^2 \int_{0}^{t} R_n(\beta_t)ds + (\gamma Tr(\Bar{X}^T\Bar{X})-2)\int_{0}^{t} \norm{\Bar{X}(\theta_s-\beta_s)}^2ds
\end{split}
\end{equation*}

Finally, for $\gamma<2/Tr(\Bar{X}^T\Bar{X})$, 

\begin{equation*}
\boxed{
    \mathbf{E}(\eta_t)\leq \gamma d\sigma^2 \int_{0}^{t} R_n(\beta_s) ds = \gamma d\sigma^2 \int_{0}^{t} L(\beta_s) ds
}
\end{equation*}

and in particular, 

\begin{equation*}
\boxed{
    \mathbf{E}(\eta_\infty)\leq \gamma d\sigma^2 \int_{0}^{\infty} L(\beta_s) ds
}
\end{equation*}

\end{proof}

\section{Diagonal Linear Network}\label{appendix:DLN}

In this Section, we first derive in Section~\ref{sec:appendix_DLN_proofs} the mathematical proofs of the Propositions and Theorems presented in Section~\ref{sec:DLN}.  
We then show more empirical results on the sparsity 
 of the solutions obtained with Noisy-SGD, when training a DLN.

\subsection{Mathematical proofs}\label{sec:appendix_DLN_proofs}









\subsubsection{Modelisation}
Recall that the weights $(w_{t,\pm})_{t\geq 0}$ are defined through the following SDE
\begin{equation*}
        dw_{t,\pm} =  \pm [\Bar{X}^T r(w_t)]\cirbd w_{t,\pm}dt + 2\sqrt{\gamma L(w_t)}w_{t,\pm} \cirbd \Bar{X}^TdB_{t,\pm}  + 2\sigma\sqrt{\gamma L(w_t)}w_{t,\pm} \cirbd d\Tilde{B}_{t,\pm}
\end{equation*}
The Euler discretization of that SDE with step size $\gamma$ is exactly
\begin{equation*}
        w_{t+1,\pm} = w_{t,\pm} - \gamma\nabla_{w_{t,\pm}} L(w_t) \pm 2\sqrt{\gamma L(w_t)}[\bar{X}^T\epsilon]\cirbd w_{t, \pm}   \pm \gamma \sigma_t Z_{t,\pm} \cirbd w_{t, \pm},
\end{equation*}
Where $\epsilon \sim \mathcal{N}(0, \sqrt{\gamma}Id)$. As seen in \cite[Appendix A]{pesme2021implicit} the covariance of $ 2\sqrt{n L(w_t)}\frac{1}{\sqrt{\gamma}}\epsilon $ is equal to the one of $\zeta_{i_t}$. Where 
\begin{equation*}
    \zeta_{i_t} = -\left(\langle \beta - \beta^\ast,x_{i_t}\rangle e_{i_t}- \mathbb{E}_{i_t}[\langle \beta - \beta^\ast,x_{i_t}\rangle e_{i_t}]\right)
\end{equation*}
with $i_t$ the random index chosen for the step of SGD. Thus the discretization scheme defines a Markov-Chain whose noise is the one found in equation \eqref{eq:noisy_SGD_DLN} (the other term being independant from the rest).

\subsubsection{Proof of Proposition 1}
This section and the following one essentially adapt the proof of a similar result found in \cite{pesme2021implicit} in order to generalize it to more complex noise pattern. In particular extend it to a noise which does not depend on the loss and thus the geometry of the gradients.
We consider throughout this section a more general setting. Let $A \in \mathbb{R}^{p\times d}$ which in the main text is either equal to $\bar{X}$ and thus $p=n$ or to $\left(\bar{X}\mid \sigma I_d\right)$ and $p = d+n$. And let a covariance noise matrix $(\sigma_t)$ which is deterministic and valued in $\mathbb{R}^{p'\times d}$. Set $(w_{t,\pm})$ solutions of the following SDE with $w_{0,\pm} = \alpha$.
\begin{equation*}
\begin{split}
    d w_{t,\pm} &= - \nabla_{w_\pm}L(\beta_t)dt \pm 2\sqrt{\gamma L(\beta_t)} w_{t,\pm} \cirbd \left[A^T dB_t\right]\pm 2w_{t,\pm}\cirbd\left[\sigma_t^Td\tilde{B}_t\right]\\
    &= \pm \left(-\bar{X}^T r(\beta_t) \cirbd w_{t,\pm} dt+ 2\sqrt{\gamma L(\beta_t)} w_{t,\pm} \cirbd \left[A^TdB_t\right]+ 2 w_{t,\pm} \cirbd \left[\sigma_t^Td\tilde{B}_t\right]\right)
\end{split}
\end{equation*}
where $r(w) = \bar{X}(w_+^2-w_-^2- \beta^\ast)$ and $(B_t)_t,(\tilde{B}_t)_t$ are two independant standard brownian motions respectively valued in $\mathbb{R}^p$ and $\mathbb{R}^{p'}$. This setting accounts for a decaying noise which has a geometry different to the one of the data as well a as a small noise which is solely of the magnitude of the parameters.
\begin{lemma}\label{lem:formule_beta_sinh}
    Consider the iterates $(w_t)$ defined above \eqref{eq:DLN_SGD_optim} then $(\beta_t = w_{t,+}^2-w_{t,-}^2)$ satisfies
    \begin{equation*}
        \beta_t =   2\alpha_t^2 \text{sinh}(2\eta_t + 2\delta_t)
    \end{equation*}
    where 
\begin{equation*}
    \eta_t =-\int_0^t \bar{X}^Tr(w_s)ds + 2\sqrt{\gamma}\int_0^t\sqrt{L(w_s)}A^TdB_{s}, \quad 
    \delta_t = 2 \int_0^t \sigma_s^Td\Tilde{B}_{s}
\end{equation*}
and
\begin{equation*}
    \alpha_t = \alpha \cirbd \exp\left(-2\gamma \text{diag}(A^TA) \int_0^t L(\beta_s)ds -2 \int_0^t \text{diag}(\sigma_t^T\sigma_t)ds\right)
\end{equation*}
\end{lemma}
\begin{proof}
A direct application of Itô's lemma grants 
\begin{equation*}
\begin{split}
    w_{t,\pm} &= w_{0,\pm} \cirbd \exp(\pm\left[-\int_0^t \bar{X}^Tr(w_s)ds + 2\sqrt{\gamma}\int_0^t\sqrt{L(w_s)}A^TdB_{s}+2\int_0^t \sigma_s^T d\tilde{B}_s\right])\\
    &\cirbd \exp(-2\gamma \text{diag}(A^TA) \int_0^t L(\beta_s)ds -2 \int_0^t \text{diag}(\sigma_t^T\sigma_t)ds)
\end{split}
\end{equation*}
Thus since $\beta_t = w_{t,+}^2-w_{t,-}^2$ we have
\begin{equation*}
\begin{split}
    \beta_t &= \alpha_t^2 \cirbd (\exp(+2\eta_t + 2\delta_t) - \exp(-2\eta_t -2\delta_t))\\
    & = 2\alpha_t^2\text{sinh}(2\eta_t + 2\delta_t)
\end{split}
\end{equation*}
\end{proof}
The next result is the generalization of the result introducing the notion of mirror gradient descent with varying potential found in \cite[Proposition 1]{pesme2021implicit}.
\begin{proposition}\label{prop:mirror_descent}
    The flow $(\beta_t)_{t\geq0}$ associated to $(w_{t,\pm})_{t\geq 0}$ follows a stochastic mirror gradient with varying potential defined by:
    \begin{equation*}
        d\nabla\phi_{\alpha_t}(\beta_t) = - \nabla L(\beta_t) dt + \sqrt{\gamma L(\beta_t)} A^T dB_t + \sigma_t^T d\tilde{B}_t 
    \end{equation*}
\end{proposition}

\begin{proof}
The expression of $\beta_t$ from lemma \ref{lem:formule_beta_sinh} can be inversed in order to use $\phi_{\alpha_t}$. Indeed
\begin{equation*}
    \text{arcsinh}\left(\frac{\beta_t}{2\alpha_t^2}\right) = 2 \eta_t + 2\delta_t
\end{equation*}
which implies
\begin{equation*}
    d\text{arcsinh}\left(\frac{\beta_t}{2\alpha_t^2}\right) = -2 \Bar{X}^T r(w_t) + 4 \sqrt{\gamma L(\beta_t)} A^TdB_t + 4 \sigma_t^Td\Tilde{B}_{t}
\end{equation*}
Notice that $\nabla L(\beta_t) = \frac{1}{2} \Bar{X}^T r(w_t)$ and $\nabla\phi_\alpha(\beta) = \frac{1}{4}\text{arcsinh}(\frac{\beta}{2\alpha^2})$ we have the wanted result.
\end{proof}
\subsubsection{Proof of Theorem 2}
In order to prove the convergence of the integral of the loss we will introduce a perturbation of a Bregman divergence which will control the norm of the iterates and the integral of the loss. Moreover that process will be nicely controlled with high probability. Let $\beta^\ast$ any interpolator, the process is the following
\begin{equation*}
\begin{split}
    V_t &= - \phi_{\alpha_t}(\beta_t) + \langle \nabla\phi_{\alpha_t}(\beta_t),\beta_t -\beta^\ast \rangle\\
    &+ \langle \vert \beta^\ast \vert, \gamma \text{diag}(A^TA)\int_0^t L(\beta_s)ds + \int_0^t\text{diag}(\sigma_s^T\sigma_s)ds\rangle
\end{split}
\end{equation*}
We denote by $D_t$ the Bregman divergence part.
\begin{lemma}\label{lem:lyapunov}
    For $t\geq 0$ we have
    \begin{equation*}
        \begin{split}
    d D_t &= -2 L(\beta_t)+\gamma L(\beta_t)\langle\xi_t,\text{diag}(A^TA) \rangle dt + \langle \xi_t,\text{diag}(\sigma_t^T\sigma_t)\rangle dt \\
    &+\sqrt{\gamma L(\beta_t)} \langle A^TdB_t,\beta_t -\beta^\ast \rangle + \langle \sigma_t^Td\tilde{B}_t,\beta_t -\beta^\ast \rangle
\end{split}
    \end{equation*}
    where $\xi_t = \sqrt{\beta_t^2 + 4 \alpha_t^4} = w_{t,+}^2 + w_{t,-}^2$
\end{lemma}
\begin{proof}
 By definition $D_t = - \phi_{\alpha_t}(\beta_t) + \langle \nabla \phi_{\alpha_t}(\beta_t),\beta_t-\beta^\ast\rangle$ thus a direct application of Itô's lemma grants
\begin{equation*}
\begin{split}
    dD_t &= - \langle \nabla \phi_{\alpha_t}(\beta_t),d\beta_t \rangle - \langle \nabla_{\alpha} \phi(\alpha_t^2,\beta), d[\alpha_t^2]\rangle - \frac{1}{2}\text{Tr}\left(\nabla^2 \phi_{\alpha_t}(\beta_t)d\langle \beta\rangle_t \right) \\
    &+ \langle d\nabla \phi_{\alpha_t}(\beta_t),\beta_t - \beta^\ast \rangle + \langle \nabla \phi_{\alpha_t}(\beta_t),d\beta_t\rangle + \text{Tr}\left(d\langle\nabla \phi_{\alpha_t}(\beta_t),\beta_t\rangle\right)\\
    &=- \langle \nabla_{\alpha} \phi(\alpha_t^2,\beta), d[\alpha_t^2]\rangle - \frac{1}{2}\text{Tr}\left(\nabla^2 \phi_{\alpha_t}(\beta_t)d\langle \beta\rangle_t \right) + \langle d\nabla \phi_{\alpha_t}(\beta_t),\beta_t-\beta^\ast\rangle \\
    &+ \text{Tr}\left(d\langle\nabla \phi_{\alpha_t}(\beta_t),\beta_t\rangle\right)
\end{split}
\end{equation*}
We will compute one by one the four terms left. Note that todo so we will need to compute the quadratic variation of $\beta_t$. Since $\beta_t$ can be decomposed in a drift part and a martingale part as follows
\begin{equation*}
    d\beta_t = \text{drift}(\beta)_tdt + 4 \sqrt{\gamma L(\beta_t)}\xi_t\cirbd A^TdB_t+ 4 \xi_t \cirbd \sigma_t^T d\tilde{B}_t
\end{equation*}
we obtain this formula for the quadratic variation
\begin{equation*}
    d\langle \beta \rangle_t = 16 \gamma L(\beta_t) (A^TA)\cirbd (\xi_t\xi_t^T) +  16 (\sigma_t^T\sigma_t) \cirbd (\xi_t\xi^T_t)
\end{equation*}
Since $\alpha_t$ has no martingale part the chain rule grants that the first term is equal to
\begin{equation*}
    \langle \nabla_{\alpha} \phi(\alpha_t^2,\beta), d[\alpha_t^2]\rangle = \gamma L(\beta_t)\langle\xi_t,\text{diag}(A^TA) \rangle dt + \langle \xi_t,\text{diag}(\sigma_t^T\sigma_t)\rangle dt
\end{equation*}
Since $\nabla^2\phi_\alpha(\beta) = \frac{1}{4}\text{diag}(\frac{1}{\xi})$ the second term is equal to
\begin{equation*}
    \frac{1}{2}\text{Tr}\left(\nabla^2 \phi_{\alpha_t}(\beta_t)d\langle \beta\rangle_t \right)  = 2 \gamma L(\beta_t)\langle \xi_t,\text{diag}(A^TA)\rangle  dt + 2 \langle \xi_t, \gamma \text{diag}(\sigma_t^T\sigma_t) \rangle dt
\end{equation*}
Using proposition \ref{prop:mirror_descent} the third term is equal to
\begin{equation*}
    \langle d\nabla \phi_{\alpha_t}(\beta_t),\beta_t-\beta^\ast\rangle = - \langle \nabla L(\beta_t),\beta_t-\beta^\ast \rangle + \sqrt{\gamma L(\beta_t)} \langle A^TdB_t,\beta_t -\beta^\ast \rangle + \langle \sigma_t^Td\tilde{B}_t,\beta_t -\beta^\ast \rangle
\end{equation*}
the fourth term 
\begin{equation*}
    \text{Tr}\left(d\langle\nabla \phi_{\alpha_t}(\beta_t),\beta_t\rangle\right) = 4 \gamma L(\beta_t)\langle \xi_t,\text{diag}(A^TA)\rangle  dt + 4 \langle \xi_t, \gamma \text{diag}(\sigma_t^T\sigma_t) \rangle dt
\end{equation*}
Note that 
\begin{equation*}
    \langle \nabla L(\beta_t),\beta_t-\beta^\ast \rangle = 2 L(\beta_t)
\end{equation*}
Thus combined it grants
\begin{equation*}
\begin{split}
    d D_t &= -2 L(\beta_t)+\gamma L(\beta_t)\langle\xi_t,\text{diag}(A^TA) \rangle dt + \langle \xi_t,\text{diag}(\sigma_t^T\sigma_t)\rangle dt \\
    &+\sqrt{\gamma L(\beta_t)} \langle A^TdB_t,\beta_t -\beta^\ast \rangle + \langle \sigma_t^Td\tilde{B}_t,\beta_t -\beta^\ast \rangle
\end{split}
\end{equation*}
\end{proof}
The Bregman divergence is useful to control the iterations as the next lemma will state. Moreover under suitable assumptions on the noise $\sigma_t$ the martingale part being close to its quadratic variation will enable us to control $V_t$ itself and in turns the loss.
\begin{lemma}\label{lem:control_iterations}
    For $t \geq 0$ we have
    \begin{equation*}
        \Vert \xi_t \Vert_1 \leq 4 V_t + \Vert \beta^\ast \Vert_1 \ln\left(\frac{\sqrt{2}\Vert \xi_t\Vert_1}{\min \alpha_i^2}\right)
    \end{equation*}
\end{lemma}
\begin{proof}
A direct computation grants
\begin{equation}\label{eq:lem:control_iterations}
    \Vert \xi_t \Vert_1 = 4D_t + \langle \text{arcsinh}\frac{\beta_t}{4\alpha_t^2},\beta^\ast \rangle
\end{equation}
Since $\text{arcsinh(x)}\leq \ln(2(x+1))$ we have
\begin{equation*}
\begin{split}
    \langle \text{arcsinh}\frac{\beta_t}{4\alpha_t^2},\beta^\ast \rangle &\leq \sum_i \vert \beta^\ast_i \vert \ln \left(\frac{\vert \beta_{i,t}\vert + 2 \alpha_{i,t}^2}{\alpha_{i,t}^2}\right)\\
    &\leq \sum_i \vert \beta^\ast_i \vert \ln \left(\frac{\vert \beta_{i,t}\vert + 2 \alpha_{i,t}^2}{\min \alpha_i^2}\right) + 4 \langle \vert \beta^\ast \vert, \gamma \text{diag}(A^TA)\int_0^t L(\beta_s)ds + \int_0^t\text{diag}(\sigma_s^T\sigma_s)ds\rangle\\
    &\leq \Vert \beta^\ast \Vert_1 \ln\left(\frac{\sqrt{2}\Vert \xi_t\Vert_1}{\min \alpha_i^2}\right)+ 4 \langle \vert \beta^\ast \vert, \gamma \text{diag}(A^TA)\int_0^t L(\beta_s)ds + \int_0^t\text{diag}(\sigma_s^T\sigma_s)ds\rangle
\end{split}
\end{equation*}
Because $\alpha_t^2 = \alpha^2\cirbd \exp\left(-4\gamma \text{diag}(A^TA) \int_0^t L(\beta_s)ds -4 \int_0^t \text{diag}(\sigma_s^T\sigma_s)ds\right)$ and $\vert \beta_{i,t}\vert + 2 \alpha_{i,t}^2 \leq \sqrt{2}\Vert \xi_t \Vert_1$. Thus plugging that inequation in the first relationship \eqref{eq:lem:control_iterations} grants the wanted result.
\end{proof}
We now need to control $V_t$, however in order to only have to manage bounded variation processes we will use the fact that a Martingale is controlled by its quadratic variation
\begin{lemma}[\cite{howard2020timeuniform}, Corollary 11]\label{lem:howard}
For any locally square integrable process $(S_t)$ with a.s. continuous trajectories and any $a,b > 0$
    \begin{equation*}
        \mathbb{P}(\exists t \geq 0 , S_t \geq a + b \langle S \rangle_t) \leq \exp(-2ab)
    \end{equation*}
\end{lemma}
We can use this lemma for the following two processes
\begin{equation*}
    \int_0^t \sqrt{\gamma L(\beta_s)} \langle A^TdB_s,\beta_s -\beta^\ast \rangle, \quad \int_0^t\langle \sigma_s^Td\tilde{B}_s,\beta_s -\beta^\ast \rangle
\end{equation*}
For $a,b >0$ we have thanks to lemma \ref{lem:howard}  with probability at least $1-2\exp(-ab)$ for any $t\geq 0$
\begin{equation*}
\begin{split}
    &\left\vert \int_0^t \sqrt{\gamma L(\beta_s)} \langle A^TdB_s,\beta_s -\beta^\ast \rangle + \langle \sigma_s^Td\tilde{B}_s,\beta_s -\beta^\ast \rangle\right\vert \\
    &\leq a + b \gamma \Vert A\Vert^2 \int_0^t  L(\beta_s)(\Vert \beta_s \Vert^2 + \Vert \beta^\ast \Vert^2)ds + b \int_0^t \Vert \sigma_s \Vert^2 (\Vert \beta_s \Vert^2 + \Vert \beta^\ast \Vert^2)ds
\end{split}
\end{equation*}
because the quadratic variation of both processes are upperbounded by the term on the right. We denote by $\mathcal{A}$ the set on which the inequality above is satsfied, we recall that this set has probability at least $1-2\exp(-ab)$.
\begin{lemma}
    On $\mathcal{A}$ for $\gamma$ and $\int \Vert \sigma_s \Vert^2 ds$ small enough the iterates are bounded by a constant that only depends on $a,b,\Vert \beta^\ast\Vert,\alpha$.
\end{lemma}
\begin{proof}
Since we are on the event $\mathcal{A}$ the bound on the martingale part is true and thus
\begin{equation*}
    V_t \leq V_0 + a - 2 \int_0^t L(\beta_s)U_s ds + \int_0^t \langle \vert \beta^\ast \vert,\text{diag}(\sigma_s^T\sigma_s) \rangle + b\Vert \sigma_s \Vert^2 (\Vert \beta_s \Vert^2 + \Vert \beta^\ast \Vert^2) ds
\end{equation*}
Where $U_s = 1 - \frac{\gamma}{2}\left[\langle \text{diag}(A^TA,\xi_s+\vert \beta^\ast\vert\rangle + 2b \Vert A\Vert^2 (\Vert \beta_s \Vert^2 + \Vert \beta^\ast \Vert^2)ds)\right]$. Assume that $\int_0^\infty \Vert \sigma_s \Vert^2 ds \leq 1$, then we have
\begin{equation*}
    V_t \leq C - 2 \int_0^t L(\beta_s)U_s ds + \int_0^t C'\Vert \sigma_s \Vert^2 \Vert \beta_s \Vert^2 ds 
\end{equation*}
Where $C$ depends on $a,b,\Vert \beta^\ast \Vert_1$ and $V_0$ which depends on $\alpha$ also $C'$ depends on $b$.
In particular if $U_s \geq 0$ for $s \in [0, \tau]$ using lemma \ref{lem:control_iterations} we have for other constants $C,C'$ and $t\in [0,\tau]$ 
\begin{equation*}
    \Vert \xi_t \Vert_1 \leq \Vert \beta^\ast \Vert_1 \ln\left(\frac{\sqrt{2}\Vert \xi_t\Vert_1}{\min \alpha_i^2}\right) + C + \int_0^t C'\Vert \sigma_s \Vert \Vert \beta_s \Vert^2 ds
\end{equation*}
Since $\Vert \beta_s \Vert^2 \leq \Vert \xi_s \Vert^2$ and by the sublinear growth of the log we get 
\begin{equation*}
    \Vert \xi_t \Vert_2 \leq C + \int_0^t C' \Vert \sigma_s \Vert^2 \Vert \xi_s \Vert^2 ds 
\end{equation*}
For another set of constants $C,C'$. We now use a bootstrap argument let $I =\{t \in [0,\tau]\mid\forall s \leq t ,\Vert \xi_t \Vert_2 \leq 2C \}$. 
Assume that $\int_0^{\infty} \Vert \sigma_s \Vert^2 ds \leq 1/4CC'$ then $I = [0,\tau]$. 
Indeed $0 \in I$ and $I$ is an interval. Let $t = \sup I$, first $t\in I$ and if $t = \tau$ there is nothing to prove. If $t< \tau$ then 
\begin{equation*}
    \frac{\Vert \xi_t \Vert_2 }{2C} \leq \frac{1}{2} + \int_0^t 2CC' \Vert \sigma_s \Vert^2 \left(\frac{\Vert \xi_s \Vert_2}{2C}\right)^2ds \leq \frac{1}{2} + \int_0^t 2CC' \Vert \sigma_s^2 \Vert ds < 1
\end{equation*}
And by continuity of $\xi$ we have $t+ \epsilon \in I$. Thus $I = [0,\tau]$ and for any $t\in [0,\tau]$ we have that $\xi$ is bounded by a constant which only depends on the parameters of the problem. Now we shall see that for a nice choice of $\gamma$ we have $\tau = + \infty$ on $\mathcal{A}$. Indeed for $\gamma$ small enough we have $U_0 \geq 1/2$. Now let $t$ the waiting time for $U$ to reach $1/2$. Then by continuity of $U$ for $s\leq t$ we have $U_s \geq 1/2 >0$ thus the bound on the iterates found above is valid and for $\gamma$ small enough we have that $U_t > 1/2$ which is a contradiction by continuity of $U$. Thus there is a threshold under which $U$ is always greater than $1/2$. Thus the above bound on the iterates is valid for all times.
\end{proof}
\begin{proposition}
    On $\mathcal{A}$ for $\gamma$ and $\int \Vert \sigma_s \Vert^2 ds$ small enough the integral of the loss converges.
\end{proposition}
\begin{proof}
In the proof of the last lemma we have shown that for the right choice of $\gamma$ and $\int \Vert \sigma_s \Vert^2 ds$, $U_s$ is greater than $\frac{1}{2}$ on $\mathcal{A}$. Thus by boundedness of the iterates we have
\begin{equation*}
    \int_0^t L(\beta_s)ds \leq C - V_t + C' \int_0^t \Vert \sigma_s \Vert^2 ds
\end{equation*}
It remains to lower bound $V$ in order to prove convergence of the integral of the loss. Or the computations in lemma \ref{lem:control_iterations} shows that
\begin{equation*}
    V_t \geq \frac{1}{4}\left(\Vert \xi_t \Vert_1 - \Vert \beta^\ast \Vert_1 \ln\left(\frac{\sqrt{2}\Vert \xi_t\Vert_1}{\min \alpha_i^2}\right)\right)
\end{equation*}
which is lower bounded because the iterates are bounded. Thus the integral converges.
\end{proof}
Finally we have all the ingredients to show that the iterates converge.
\begin{proposition}
    On $\mathcal{A}$ for $\gamma$ and $\int \Vert \sigma_s \Vert^2 ds$ small enough the iterates converge.
\end{proposition}
\begin{proof}
We have the following integral expression of $D_t$ thanks to lemma \ref{lem:lyapunov}
\begin{equation*}
\begin{split}
    D_t &= D_0+ \int_0^t -2 L(\beta_s)+\gamma L(\beta_s)\langle\xi_s,\text{diag}(A^TA) \rangle + \langle \xi_s,\text{diag}(\sigma_s^T\sigma_s)\rangle ds \\
    &+\int_0^t \sqrt{\gamma L(\beta_s)} \langle \bar{X}^TdB_s,\beta_s -\beta^\ast \rangle + \langle \sigma_sd\tilde{B}_s,\beta_s -\beta^\ast \rangle ds
\end{split}
\end{equation*}
Note that $D_t$ converges because the bouded variation part converges due to the convergence of the integral of the loss and the boundedness of the iterates. The martingale part converges because the quadratic variation is converging. Recall that
\begin{equation*}
    D_t = -\phi_{\alpha_t} + \langle \nabla \phi_{\alpha_t}(\beta_t), \beta_t - \beta^\ast \rangle
\end{equation*}
We have that $D$ converges for any choice of interpolator $\beta^\ast$.\\
The integral of the loss is convergent thus up to extraction $L(\beta_t)$ converges to 0. Since $\beta_t$ is bounded a second extraction ensures the converges of the iterates. Thus there is a subsequence $\beta_{\psi(t)}$ such that it converges to $\beta_\infty$ and $L(\beta_\infty) = 0$ thus it is an interpolator. And we have
\begin{equation*}
\begin{split}
    \phi_{\alpha_\infty}(\beta_\infty) - D_t &= \phi_{\alpha_\infty}(\beta_\infty) - \phi_{\alpha_t}(\beta_t) - \langle \nabla\phi_{\alpha_t}(\beta_t),\beta_t - \beta_\infty\rangle\\
    &\geq \phi_{\alpha_t}(\beta_\infty) - \phi_{\alpha_t}(\beta_t) - \langle \nabla\phi_{\alpha_t}(\beta_t),\beta_t - \beta_\infty\rangle\\
    &= D_{\phi_{\alpha_t}}(\beta_\infty,\beta)\\
    &\geq 0
\end{split}
\end{equation*}
Because $\alpha \mapsto \phi_{\alpha}(\beta)$ is decreasing and $\alpha_t \geq \alpha_\infty$.
Note that $\phi_{\alpha_\infty}(\beta_\infty) - D_{\psi(t)} \to 0$ by convergence of $\beta_{\psi(t)}$ to $\beta_\infty$. Thus by convergence of $D$ we have that $\phi_{\alpha_\infty}(\beta_\infty) - D_t \to 0$ and in turn $D_{\phi_{\alpha_t}}(\beta_\infty,\beta) \to 0$. Recall that
\begin{equation*}
    \nabla^2 \phi_{\alpha}(\beta) = \frac{1}{4}\text{diag}\left(\frac{1}{\sqrt{\beta^2+4\alpha^2}}\right) \geq \frac{1}{4 \sqrt{ \Vert \beta \Vert^2 + 4\max \alpha_i^2}}Id
\end{equation*}
 thus $\phi_{\alpha_t}$ is strongly convex of with a parameter that does not depend on $t$ by decreasingness of $\alpha_t$ and boundedness of $\beta_t$. Thus we have that $D_{\phi_{\alpha_t}}(\beta_\infty,\beta) \geq \mu/2 \Vert \beta_t - \beta_\infty \Vert^2$ which concludes the convergence.
\end{proof}
\subsubsection{Proof of Proposition 3}
In the last section we have shown that the iterates converge toward an interpolator. However that interpolator is not a minimizer of the potential $\phi_{\alpha_\infty}$ as shown in proposition \ref{prop:min_problem_DLN}. However by strong convexity of $\phi_\alpha$ we have a control of the distance between $\beta_\infty$ and the minimizer of $\phi_{\alpha_\infty}$ which we denote by $\beta^\ast_{\alpha_\infty}$
\begin{proof}
Let $\phi(\beta) = \phi_{\alpha_\infty}(\beta) + \iota_{X\beta = Y}$ where $\iota$ is the indicator function equal to $0$ if the condition is satisfied and $+\infty$ otherwise. 
Note that $\phi$ is still strongly convex. 
Indeed, we can show that $\phi - \mu\Vert.\Vert^2$ is convex: let $\beta,\nu \in \mathbb{R}^d$, $\lambda \in (0,1)$. $C = \{\beta \text{ s.t. } X\beta=Y\}$ is a convex set as an affine subspace of $\mathbb{R}^d$. Therefore:

\begin{equation}
    \lambda \beta + (1-\lambda)\nu \notin C \implies \beta \notin C \text{ or } \nu \notin C 
\end{equation}
\begin{equation}
    \iota_{X(\lambda \beta + (1-\lambda)\nu)=Y} = \infty \implies  \iota_{X\beta=Y} = \infty \text{ or } \iota_{X\nu=Y} = \infty
\end{equation}

We also have that $\phi_{\alpha_\infty}-\mu\Vert.\Vert^2$ is convex on a bounded set which contains $\beta_\infty$ and $\beta_{\alpha_\infty}^\ast$. Finally $\phi - \mu\Vert.\Vert^2$ is convex, i.e. $\phi$ is $\mu-$strongly convex on a bounded set.

We have $\beta_\infty$ solution of 
\begin{equation}
    \min \phi(\beta) -\langle r_\infty , \beta\rangle
\end{equation}
And $\beta_{\alpha_\infty}^\ast$ is the solution to
\begin{equation}
    \min \phi(\beta)
\end{equation}


Since $\phi$ is strongly convex we have that $f = \phi - \mu \Vert \cdot - \beta^\ast \Vert^2$ is convex. Notice that $0 \in \partial f(\beta_\infty) - r_\infty$ thus
\begin{equation}
    \phi(\beta^\ast) -\langle r_\infty , \beta_0\rangle- \mu \Vert \beta^\ast - \beta_\infty \Vert^2 \geq \phi(\beta_\infty) -\langle r_\infty , \beta_\infty\rangle 
\end{equation}

We thus observe that:
\begin{equation}
    \phi(\beta^\ast) -\langle r_\infty , \beta^\ast\rangle \geq \phi(\beta_\infty) -\langle r_\infty , \beta_\infty\rangle + \mu \Vert \beta^\ast - \beta_\infty \Vert^2
\end{equation}

Thus by optimality of $\beta^\ast$ ($\phi(\beta^\ast) - \phi(\beta_\infty) \leq 0$)

\begin{equation}
\begin{split}
    \langle r_\infty ,\beta_\infty- \beta^\ast\rangle & \geq \phi(\beta^\ast) -\langle r_\infty , \beta^\ast\rangle -(\phi(\beta_\infty) -\langle r_\infty , \beta_\infty\rangle) \\ &\geq \mu \Vert \beta^\ast - \beta_\infty \Vert^2
\end{split}
\end{equation}

Finally using Cauchy-Schwarz inequality we have

\begin{equation}\label{eq:appendix_final_ineq}
\boxed{
    \Vert r_\infty \Vert \geq \mu \Vert \beta^\ast - \beta_\infty \Vert
}
\end{equation}

\end{proof}

\subsection{Additional Numerical Experiments}\label{sec:appendix_additional_exps}

\paragraph{Reminder}We adopt the same set-up as \cite{pesme2021implicit} for sparse regression.
We select parameters n = 40 and d = 100, and then create a sparse model $\beta^*_{l_0}$ with the constraint that its $l_0$ norm is equal to 5. 
We generate the features $x_i$ from a normal distribution with mean 0 and identity covariance matrix $N(0, I)$, and compute the labels as $y_i = x_i^T \beta^*_{l_0}$. 
We always use the same step size of $\gamma = 1/(1.3||\Bar{X}^T\Bar{X}||_2)$. Notice that $||\beta_t - \beta^*_{l_0}||^2_2$ is the validation loss in the experiments.

\paragraph{Introduction}As outlined in Section~\ref{sec:DLN} and prooved in Section~\ref{sec:appendix_DLN_proofs} in a more general setting, our investigation reveals that Noisy-SGD yields an identical solution to GD, albeit (1) operating on an altered potential $\phi_{\alpha_{\infty}}$, and (2) starting from an effective non-zero initialization denoted as $\Tilde{\beta}_0$. 
In contrast, the sole parameter affected by SGD is $\alpha$.
In the context of Noisy-SGD, it is noteworthy that $\alpha_{\infty}$ diminishes with increasing $\sigma$.
This signifies that as more noise is introduced, the effective $\alpha$ becomes smaller. 
If this did not influence the effective initialization, it would directly imply an ``improved" implicit bias, in the case of the sparse regression under study.


\paragraph{Trade off}To be more precise, if the effective initialization $\Tilde{\beta_0}$ remained unaltered, the noisy process would ultimately converge to $\beta^\star_{\alpha_\infty}$, which is the solution that minimizes $\phi_{\alpha_\infty}$. 
We already know that this solution is sparser than the one obtained via SGD because $\alpha_\infty$ is smaller (as referenced in Equation~\ref{eq:alpha_noisy}).
As demonstrated in Proposition~\ref{eq:impact_eff_init}, the presence of the new effective initialization $\Tilde{\beta_0}$ does introduce some perturbation to the solution. 
However, the degree of its influence directly depends on the level of noise introduced. 
This ensures that $\beta_{\infty}$, which is the solution obtained through Noisy-SGD, remains in close proximity to $\beta^\star_{\alpha_\infty}$. 
This scenario implies the existence of a tradeoff. Increasing the value of $\sigma$ will:

\begin{enumerate}
    \item Augment the sparsity of $\beta^\star_{\alpha_\infty}$ as $\beta^\star_{\alpha_\infty}$ gets closer to the (sparse) ground truth $\beta^\star_{l_0}$
    \item Expand the gap between $\beta_{\infty}$ and $\beta^\star_{\alpha_\infty}$, which could potentially diminish the impact of the first outcome if $\beta_{\infty}$ gets far from $\beta^\star_{l_0}$
\end{enumerate}

In summary, if the influence of \textbf{1.} surpasses the effect of \textbf{2.}, Noisy-SGD would enhance the implicit bias. Conversely, if \textbf{2.} outweighs \textbf{1.}, this enhancement may not necessarily occur.
To empirically quantify the trade-off, we measure the gap between $\beta^\star_{\alpha_\infty}$ and $\beta_{\infty}$ across varying values of $\alpha$ and $\sigma$. 
Our approach involves (a) obtaining $\beta_{\infty}$ through Noisy-SGD, (b) approximating $\alpha_\infty$ numerically based on the loss logs, and (c) subsequently executing GD with $\alpha = \alpha_\infty$ to obtain $\beta^\star_{\alpha_\infty}$.
In Figure~\ref{fig:Figure4}, our findings underscore that, for $\alpha=0.1$, the distance between $\beta^\star_{\alpha_\infty}$ and $\beta_{\infty}$ seem to indeed increase with growing values of $\sigma$, validating the second point (\textbf{2.}).
Moreover, as revealed in Figure~\ref{fig:Figure3}, when using $\alpha=0.1$, Noisy-SGD consistently converges to sparser solutions compared to standard SGD. 
This trend remains consistent across various noise levels, as illustrated in Figure~\ref{fig:appendix_0.1}. 
Specifically, as we elevate the magnitude of $\sigma$, the solutions become progressively closer to the sparse interpolator (ground truth).
In this context, the influence of the new initialization $\Tilde{\beta_0}$ appears to be of secondary importance, with therefore point \textbf{1.} consistently holding more significance than point \textbf{2.}

\begin{figure*}[t!]
    \centering
    \includegraphics[width=0.7\textwidth]{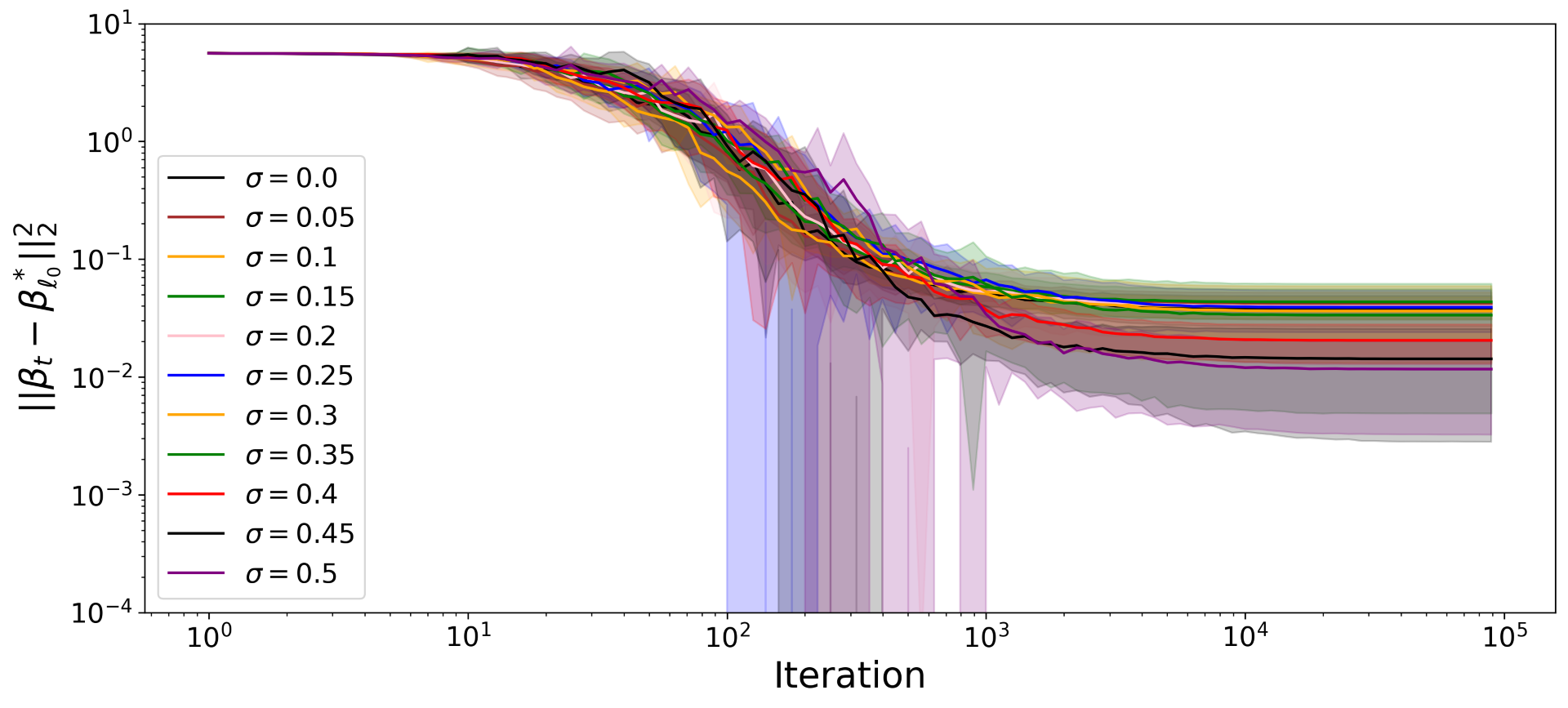}
    \vspace{-1em}
    \caption{
    Diagonal Linear Network from $\alpha=0.1$: Implicit Bias of SGD and Noisy SGD different values of $\sigma$ in Equation~\ref{eq:noisy_SGD_DLN}, \textbf{from $\alpha=0.1$}.
    Shaded areas represent one standard variation over 5 runs, and plain lines represent the average values.
    Starting from this initialization, the bigger $\sigma$ is, the closer the solution obtained with Noisy-SGD is to the sparse solution $\beta^*_{l_0}$. 
 \vspace{-0.1em}
 }
    \label{fig:appendix_0.1}
\end{figure*}

Nonetheless, Figure~\ref{fig:Figure4} reveals that the magnitude of the gap between $\beta^\star_{\alpha_\infty}$ and $\beta_{\infty}$ is smaller or comparable to the gap between $\beta_{\infty}$ and $\beta^\star_{l_0}$ in Figure~\ref{fig:Figure3}. 
This observation implies that Noisy-SGD may be advantageous only in this specific scenario, as the impact of \textbf{2.} is small. 
However, a question arises: if the initial $\alpha$ were even smaller, would \textbf{1.} still dominate \textbf{2.}?

Figure~\ref{fig:appendix_0.01} provides insight into this inquiry. 
Even with an initial $\alpha$ reduced by a factor of ten, we observe that the introduction of noise continues to exert a positive influence on the sparsity of the solution.
Furthermore, when examining the gap between $\beta^\star_{\alpha_\infty}$ and $\beta_{\infty}$ in Figure~\ref{fig:appendix_0.01_sigma}, we find that, similarly to the case where $\alpha=0.1$, it remains of a comparable order of magnitude to the gap between $\beta_{\infty}$ and $\beta^\star_{l_0}$ observed in Figure~\ref{fig:appendix_0.01}. 
This observation helps to elucidate why the advantages of Noisy-SGD persist under these conditions: the distance between $\beta^\star_{\alpha_\infty}$ and $\beta_{\infty}$ empirically also depends (inversely) on $\alpha$, which mitigates the impact of \textbf{2.} over \textbf{1.}


\begin{figure*}[t!]
    \centering
    \includegraphics[width=0.75\textwidth]{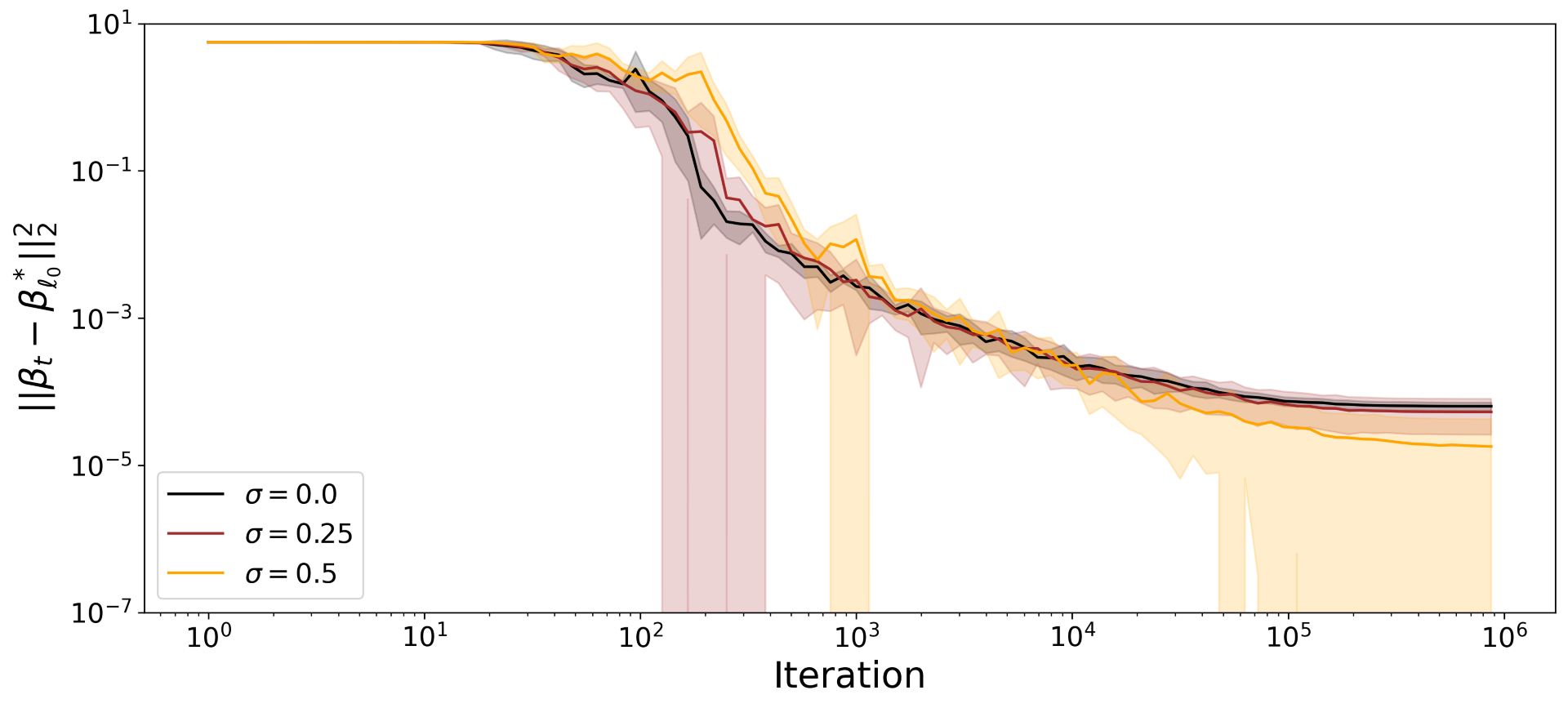}
    \vspace{-1em}
    \caption{
    Diagonal Linear Network from $\alpha=0.01$: Implicit Bias of SGD and Noisy SGD different values of $\sigma$ in Equation~\ref{eq:noisy_SGD_DLN}, \textbf{from small initialization $\alpha=0.01$}.
    Shaded areas represent one standard variation over 5 runs, and plain lines represent the average values.
    Starting from this initialization, the bigger $\sigma$ is, the closer the solution obtained with Noisy-SGD is to the sparse solution $\beta^*_{l_0}$. 
 \vspace{-0.01em}
 }
    \label{fig:appendix_0.01}
\end{figure*}

\begin{figure}[t!]
    \centering
    \includegraphics[width=0.5\columnwidth]{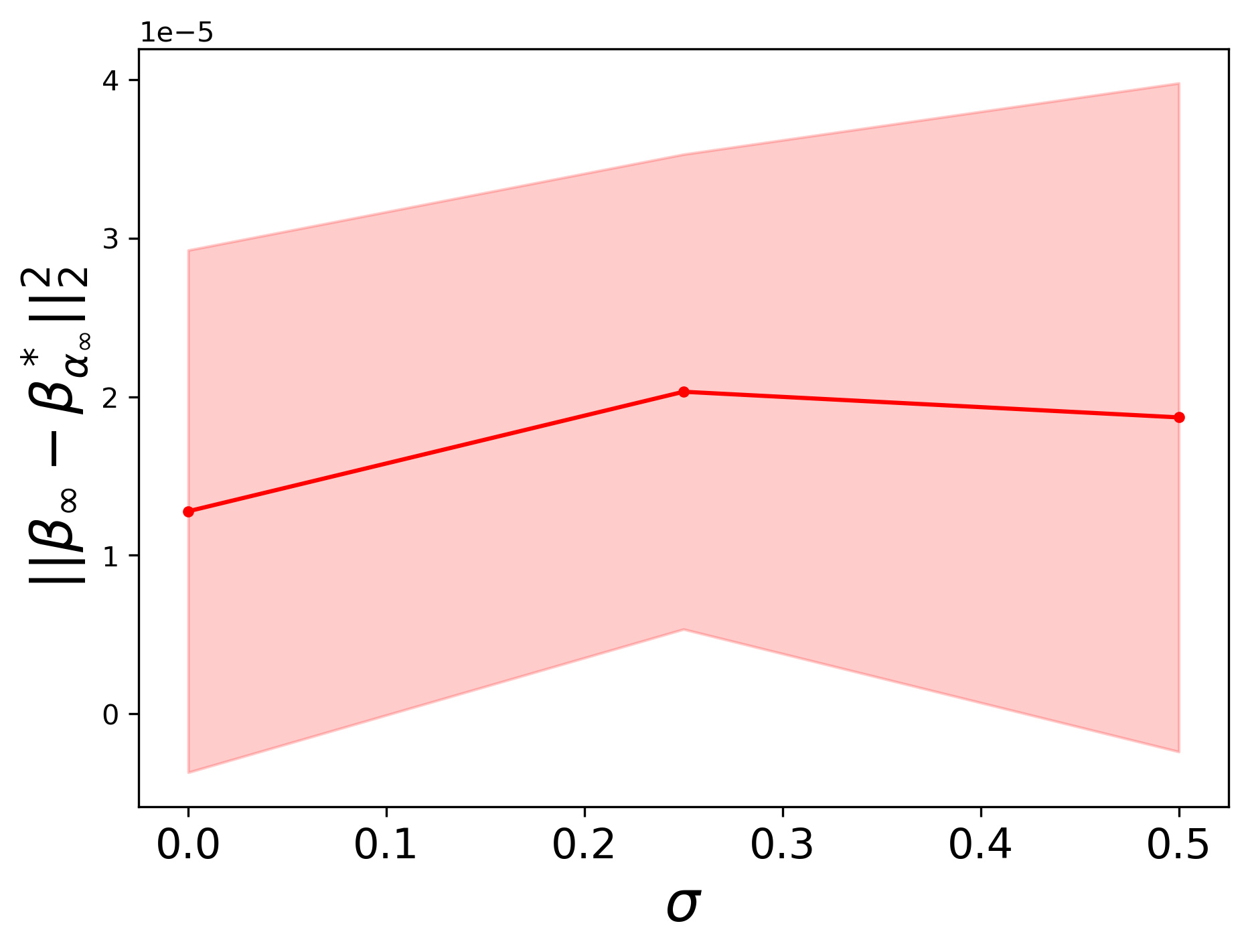}
    \vspace{-0.01em}
    \caption{
     DLN from $\alpha=0.1$: Distance between $\beta^*_{\alpha_\infty}$, the solution that minimizes $\phi_{\alpha_\infty}$ ---obtained by GD from $\alpha_\infty$--- and $\beta_{\infty}$, the one obtained by Noisy-SGD \textbf{from $\alpha=0.01$} (see Proposition~\ref{eq:impact_eff_init}).
    Shaded areas represent one standard deviation over 5 runs. 
    Similar to the case of $\alpha=0.1$ (see Figure~\ref{fig:Figure4}), the distance is of the same order of magnitude than the distance between $\beta_{\infty}$ and the sparse solution $\beta_{l_0}$ (see Figure~\ref{fig:appendix_0.01}). 
    It explains why why the implicit bias is enhanced by Gaussian noise in this case.
    \vspace{-0.01em}
 }
    \label{fig:appendix_0.01_sigma}
\end{figure}

\paragraph{Conclusion:}
The effect of Noisy-SGD on the solution's sparsity, as discussed in Section~\ref{sec:DLN}, remains significant even when initiated with smaller values of $\alpha$, which corresponds to cases where the solutions obtained from GD and SGD already exhibit some enhanced degree of sparsity. 
Therefore, the practical implications of the theoretical trade off appears to be modest, and the reinforcement of the implicit bias from Noisy-SGD prevails: the introduction of Gaussian noise enhances or at least maintains the implicit bias of SGD in a robust way.

\label{sec:appendix}

\end{document}